\documentclass{article}
\usepackage[accepted]{icml2024}
\usepackage[subfigure]{tocloft}
\usepackage{amsmath}
\usepackage{amssymb}
\usepackage{mathtools}
\usepackage{amsthm}
\usepackage{float}
\usepackage{microtype}
\usepackage{graphicx}
\usepackage{subfigure}
\usepackage{booktabs} % for professional tables
\usepackage{url}
\usepackage{graphicx}
\usepackage{nicefrac}
\usepackage{ifthen}
\usepackage{esvect}
\usepackage{multirow}
\usepackage[colorlinks,linkcolor=black,citecolor=blue, pagebackref=true]{hyperref}

\renewcommand*{\backrefalt}[4]{%
    \ifcase #1 \footnotesize{(Not cited.)}%
    \or        \footnotesize{(Cited on page~#2.)}%
    \else      \footnotesize{(Cited on pages~#2.)}%
    \fi}
\usepackage{bookmark}

\usepackage[capitalize,noabbrev]{cleveref}

\theoremstyle{plain}
\newtheorem{theorem}{Theorem}[section]

\newtheorem{lemma}[theorem]{Lemma}

\theoremstyle{definition}
\newtheorem{definition}[theorem]{Definition}

\theoremstyle{remark}

\definecolor{first}{HTML}{029E73}
\definecolor{second}{HTML}{0173B2}
\definecolor{third}{HTML}{D55E00}
\definecolor{fourth}{HTML}{8000ff}
\definecolor{fifth}{HTML}{8080ff}

\usepackage[textsize=tiny]{todonotes}

\icmltitlerunning{Molecular Conformer Aggregation Networks}

\usepackage{amsmath,amsfonts,bm}

\def\eqref#1{equation~(\ref{#1})}
\def\1{\bm{1}}

\def\vpi{{\bm{\pi}}}

\def\va{{\bm{a}}}

\def\ve{{\bm{e}}}
\def\vf{{\bm{f}}}
\def\vg{{\bm{g}}}
\def\vh{{\bm{h}}}

\def\vr{{\bm{r}}}

\def\vx{{\bm{x}}}

\def\mA{{\bm{A}}}
\def\mB{{\bm{B}}}
\def\mC{{\bm{C}}}

\def\mG{{\bm{G}}}
\def\mH{{\bm{H}}}

\def\mL{{\bm{L}}}
\def\mM{{\bm{M}}}

\DeclareMathAlphabet{\mathsfit}{\encodingdefault}{\sfdefault}{m}{sl}
\SetMathAlphabet{\mathsfit}{bold}{\encodingdefault}{\sfdefault}{bx}{n}
\newcommand{\tens}[1]{\bm{\mathsfit{#1}}}

\def\tL{{\tens{L}}}

\def\gO{{\mathcal{O}}}

\def\sR{{\mathbb{R}}}

\newcommand{\ba}{\ensuremath{{\bf a}}}
\newcommand{\bb}{\ensuremath{{\bf b}}}
\newcommand{\cP}{\mathcal{P}}

\newcommand{\cS}{\mathcal{S}}

\newcommand{\Rb}{\mathbb{R}}

\newcommand{\hb}{\mathbf{h}}

\newcommand{\eb}{\mathbf{e}}
\newcommand{\mb}{\mathbf{m}}
\newcommand{\bbr}{\mathbf{r}}

\newcommand{\Hb}{\ensuremath{\mathbf{H}}}
\newcommand{\Wb}{\ensuremath{\mathbf{W}}}
\newcommand{\Ab}{\ensuremath{\mathbf{A}}}

\newcommand{\conan}{\textsc{ConAN}}

\newcommand{\rbf}{\mathtt{RBF}}

\newcommand{\st}{\mathsf{t}}

\newcommand{\UPD}{\mathsf{UPD}}
\newcommand{\AGG}{\mathsf{AGG}}
\newcommand{\RO}{\mathsf{READOUT}}

\renewcommand{\vec}[1]{\mathbf{#1}}
\newcommand{\oms}{\{\!\!\{}
\newcommand{\cms}{\}\!\!\}}
\newcommand{\tup}[1]{{(#1)}}

\newcommand{\bbv}{\mathbf{v}}

\usepackage{bm}

\newcommand{\bM}{\mathbf{M}}

\newcommand{\be}{\bm{e}}

\definecolor{cb-black}      {RGB}{  0,   0,   0}
\definecolor{cb-blue-green} {RGB}{  0,  073,  073}
\definecolor{cb-rose}       {RGB}{255, 109, 182}
\definecolor{cb-salmon-pink}{RGB}{255, 182, 119}
\definecolor{cb-purple}     {RGB}{ 73,   0, 146}
\definecolor{cb-blue}       {RGB}{ 0, 109, 219}
\definecolor{cb-lilac}      {RGB}{182, 109, 255}
\definecolor{cb-blue-sky}   {RGB}{109, 182, 255}
\definecolor{cb-blue-light} {RGB}{182, 219, 255}
\definecolor{cb-burgundy}   {RGB}{146,   0,   0}
\definecolor{cb-brown}      {RGB}{146,  73,   0}
\definecolor{cb-clay}       {RGB}{219, 209,   0}
\definecolor{cb-green-lime} {RGB}{ 36, 255,  36}
\definecolor{cb-yellow}     {RGB}{255, 255, 109}

\definecolor{cb-green-sea}{HTML}{0173B2}
\definecolor{cb-burgundy}{HTML}{029E73}
\definecolor{cb-lilac}{HTML}{D55E00}

\definecolor{first}{HTML}{0173B2}
\definecolor{second}{HTML}{029E73}
\definecolor{third}{HTML}{D55E00}
\definecolor{fourth}{HTML}{00427e}

\definecolor{sns_cb1}{HTML}{0173B2}
\definecolor{sns_cb2}{HTML}{029E73}
\definecolor{sns_cb3}{HTML}{D55E00}
\definecolor{sns_cb4}{HTML}{CC78BC}
\definecolor{sns_cb5}{HTML}{ECE133}
\definecolor{sns_cb6}{HTML}{56B4E9}

\def\va{{\bm{a}}}

\def\ve{{\bm{e}}}
\def\vf{{\bm{f}}}
\def\vg{{\bm{g}}}
\def\vh{{\bm{h}}}

\def\vr{{\bm{r}}}

\def\vx{{\bm{x}}}

\renewcommand{\vec}[1]{\mathbf{#1}}
\usepackage{bm}
\def\mA{{\bm{A}}}
\def\mB{{\bm{B}}}
\def\mC{{\bm{C}}}

\def\mG{{\bm{G}}}
\def\mH{{\bm{H}}}

\def\mL{{\bm{L}}}
\def\mM{{\bm{M}}}

\usepackage{comment}

\newcommand\deq{\stackrel{\text{\tiny def}}{=}}

\def\st{{\em s.t.~}}
\def\ie{{\em i.e.,~}}

\def\wrt{{\em w.r.t.~}}
\def\resp{{\em resp.~}}

\newcommand{\fgwpa}{\text{FGW}_{p,\alpha}}
\newcommand{\fgwtwoa}{\text{FGW}_{2,\alpha}}
\newcommand{\wap}{\text{W}_{p}}
\newcommand{\w}{\text{W}}
\newcommand{\gw}    {\text{GW}}
\newcommand{\gwp}    {\text{GW}_{p}}
\DeclareMathOperator*{\supp}{supp}
\DeclareMathOperator*{\argmin}{arg~min}

\newcommand{\nn}{\nonumber} % Nonumber in equation
\newcommand{\R}{\mathbb{R}}
\newcommand{\Rp}{\mathbb{R}_{+}} % Define real number.
\newcommand{\I}{\mathbf{1}}
\newcommand{\Ns}{\mathbb{N}^\star} % Define natural number without 0.
\newcommand{\N}{\mathbb{N}} % Define natural number
\newcommand{\X}{\mathbb{X}} 
\newcommand{\Sb}{\mathbb{S}}

\newcommand{\E}[2]{\mathbb{E}_{#1} \left(#2\right)}
\newcommand{\bbS}{\mathbb{S}}
\newcommand{\bbA}{\mathbb{A}}

\newcommand{\bspi}{\boldsymbol{\pi}}
\newcommand{\bsPi}{\boldsymbol{\Pi}}

\newcommand{\bsomega}{\boldsymbol{\omega}}
\newcommand{\bsOmega}{\boldsymbol{\Omega}}

\newcommand{\bslambda}{\boldsymbol{\lambda}}

\begin{document}

\twocolumn[
\vspace{-0.15in}
\icmltitle{\fontsize{14}{16}\selectfont Structure-Aware 
 E(3)-Invariant Molecular Conformer Aggregation Networks}

% E(3)-Invariant Conformer Ensemble Networks with Geometry-Aware Aggregation

% E(3)-Invariant Molecular Conformer Aggregation Networks



% \usepackage[accepted]{icml2024}

\icmlsetsymbol{equal}{*}
\vspace{-0.1in}
\begin{icmlauthorlist}
\icmlauthor{Duy M. H. Nguyen}{equal,uni-stuttgart,imprs,dfki}
\icmlauthor{Nina Lukashina}{equal,uni-stuttgart,imprs}
\icmlauthor{Tai Nguyen}{dfki}
\icmlauthor{An T. Le}{uni-damstadt}
\icmlauthor{TrungTin Nguyen}{queensland}
\icmlauthor{Nhat Ho}{texas}
\icmlauthor{Jan Peters}{dfki,uni-damstadt,hessian-ai}
\icmlauthor{Daniel Sonntag}{dfki,uni-oldenburg}
\icmlauthor{Viktor Zaverkin}{nec}
\icmlauthor{Mathias Niepert}{uni-stuttgart,imprs,nec}
\end{icmlauthorlist}
\icmlaffiliation{uni-stuttgart}{Department of Computer Science, University of Stuttgart, Germany}
\icmlaffiliation{imprs}{Max Planck Research School for Intelligent Systems (IMPRS-IS)}
\icmlaffiliation{dfki}{German Research Center for Artificial Intelligence (DFKI)}
\icmlaffiliation{uni-damstadt}{Department of Computer Science, Technische Universitat Darmstadt, Germany}
\icmlaffiliation{queensland}{School of Mathematics and Physics, University of Queensland, Australia}
\icmlaffiliation{texas}{Department of Statistics and Data Sciences, University of Texas at Austin, USA}
\icmlaffiliation{texas}{Department of Statistics and Data Sciences, University of Texas at Austin, USA}
\icmlaffiliation{hessian-ai}{Hessian.AI}
\icmlaffiliation{uni-oldenburg}{Department of Applied Artificial Intelligence, Oldenburg University, Germany}
\icmlaffiliation{nec}{NEC Laboratories Europe}
\icmlcorrespondingauthor{Duy H. M. Nguyen}{hong01@dfki.de}
\icmlkeywords{Machine Learning, ICML}
\vskip 0.3in
]
\printAffiliationsAndNotice{\icmlEqualContribution}


\vspace{-1.in}
\begin{abstract}
\vspace{-0.04in}
A molecule's 2D representation consists of its atoms, their attributes, and the molecule's covalent bonds. A 3D (geometric) representation of a molecule is called a conformer and consists of its atom types and Cartesian coordinates. Every conformer has a potential energy, and the lower this energy, the more likely it occurs in nature. Most existing machine learning methods for molecular property prediction consider either 2D molecular graphs or 3D conformer structure representations in isolation. Inspired by recent work on using ensembles of conformers in conjunction with 2D graph representations, we propose $\mathrm{E}$(3)-invariant molecular conformer aggregation networks.  The method integrates a molecule's 2D representation with that of multiple of its conformers. 
Contrary to prior work, we propose a novel 2D--3D aggregation mechanism based on a differentiable solver for the \emph{Fused Gromov-Wasserstein Barycenter} problem and the use of an efficient conformer generation method based on distance geometry.  We show that the proposed aggregation mechanism is $\mathrm{E}$(3) invariant and propose an efficient GPU implementation. Moreover, we demonstrate that the aggregation mechanism helps to significantly outperform state-of-the-art molecule property prediction methods on established datasets. Our implementation is available at this \href{https://github.com/duyhominhnguyen/conan-fgw/tree/main}{link}.

\end{abstract}
\vspace{-0.25in}
\addtocontents{toc}{\protect\setcounter{tocdepth}{-1}}
\section{Introduction}
\vspace{-0.05in}
Machine learning is increasingly used for modeling and analyzing properties of atomic systems with important applications in drug discovery and material design~\citep{Butler2018, Vamathevan2019, Choudhary2022, Fedik2022, Batatia2023}. Most existing machine learning approaches to molecular property prediction either incorporate 2D (topological)~\citep{Kipf2017, Gilmer2017, Xu2018, Velickovic2018} or 3D (geometric) information of molecular structures~\citep{schutt2017schnet, Schuett2021, Batzner2022, batatia2022mace}. 2D molecular graphs describe molecular connectivity (covalent bonds) but ignore the spatial arrangement of the atoms in a molecule (molecular conformation). 3D graph representations capture conformational changes but are commonly used to encode an individual conformer. Many molecular properties, such as solubility and binding affinity \cite{cao2022design}, however, inherently depend on a large number of conformations a molecule can occur as in nature, and employing a single geometry per molecule limits the applicability of machine-learning models. Furthermore, it is challenging to determine conformers that predominantly contribute to the molecular properties of interest. Thus, developing expressive representations for molecular systems when modeling their properties is an ongoing challenge.

To overcome this, recent work has introduced molecular representations that incorporate both 2D molecular graphs and 3D conformers \cite{zhu2023learning}. These methods aim to encode various molecular structures, such as atom types, bond types, and spatial coordinates, leading to more comprehensive feature embeddings. The latest algorithms, including graph neural networks, attention mechanisms \cite{axelrod_molecular_2023}, and long short-term memory networks \cite{wang2024multi}, have demonstrated improved generalization capabilities in various molecular prediction tasks. 
Despite their effectiveness, these methods struggle to balance model complexity and performance and face scalability challenges mainly due to the computational cost of generating 3D conformers. These problems are exacerbated when using several conformers per system, underscoring the need for strategies to mitigate these limitations.

\textbf{Contributions.} We propose a new message-passing neural network architecture that integrates both 2D and ensembles of 3D molecular structures. 
The approach introduces a geometry-aware conformer ensemble aggregation strategy using Fused Gromov-Wasserstein (FGW) barycenters~\cite{titouan_optimal_2019}, in which interactions between atoms across conformers are captured using both latent atom embeddings and conformer structures. The aggregation mechanism is invariant to actions of the group $E(3)$---the Euclidean group in 3 dimensions---such as translations, rotations, and inversion as well as to permutations of the input conformers. To make the proposed method applicable to large-scale problems, we accelerate the solvers for the FGW barycenter problem with entropic-based techniques \cite{rioux2023entropic}, allowing the model to be trained in parallel on multiple GPUs. We also experimentally explore the impact of the number of conformers and demonstrate that, within our framework, a modest number of conformers generated through efficient distance geometry-based sampling achieves state-of-the-art accuracy. We partially explain this through a theoretical analysis showing that the empirical barycenter converges to the target barycenter at a rate of $\mathcal{O}\left(1/K\right)$, where $K$ denotes the number of conformers. Finally, we conduct a systematic evaluation of our proposed approaches, comparing their performance to state-of-the-art algorithms. The results show that our method is competitive and frequently surpasses existing methods across a variety of datasets and tasks.

\vspace{-0.1in}
\section{Background}
\vspace{-0.05in}
We first provide notations used in the paper. We note the simplex histogram with $n$-bins as $\Delta_n := \left\{\bsomega \in \Rp^n: \sum_{i} \omega_i    = 1\right\}$ and  $\bbS_n(\bbA)$ as the set of symmetric matrices of size $n$ taking values in $\bbA \subset \R$. For any $x \in \bsOmega$, $\delta_x$ denotes the Dirac measure in $x$.
Let $\cP(\bsOmega)$ be the set of all probability measures on a space $\bsOmega$. 
We denote $[K] = \{1, 2, \ldots, K\}$ for any $K \in \N$.
We denote the matrix scalar product associated with the Forbenius norm as $\langle \cdot \rangle$. The tensor-matrix multiplication will be denoted as $\otimes$, \ie given any tensor $\tL := \left(L_{ijkl}\right)$ and matrix $\mB:=(B_{kl})$, $\tL \otimes \mB$ is the matrix $\left(\sum_{kl} L_{ijkl} B_{kl}\right)_{ij}$.

A graph $G$ is a pair $(V,E)$ with \emph{finite} sets of
vertices or nodes $V$ and edges $E \subseteq \{ \{u,v\} \subseteq V \mid u \neq v \}$. We set $n \coloneqq |V|$ and write that the graph is of order $n$. For ease of notation, we denote the edge $\{u,v\}$ in $E$ by $(u,v)$ or $(v,u)$. The neighborhood of $v$ in $V$ is denoted by $N(v) \coloneqq  \{ u \in V \mid (v, u) \in E \}$ and the degree of a vertex $v$ is  $|N(v)|$. 
An attributed graph $G$  is a triple $(V, E, \ell_f)$ with a graph $(V, E)$ and (vertex-)feature (attribute) function $\ell_f \colon V \to \Rb^{1 \times d}$, for some $d \in \Ns$. Then $\ell_f(v)$ is an attribute or feature of $v$, for $v$ in $V$. When we have multiple attributes, we have a pair $\mG=(G,\mH)$, where $G = (V, E)$ and $\mH$ in $\Rb^{n\times d}$ is a node attribute matrix. For a matrix $\mH$ in $\Rb^{n\times d}$ and $v$ in $[n]$, we denote by $\mH_{v}$ in $\Rb^{1\times d}$ the $v$th row of $\mH$ such that $\mH_{v} \coloneqq \ell_f(v)$. Analogously, we can define attributes for the edges of the graph. Furthermore, we can encode an $n$-order graph $G$ via an adjacency matrix $\vec{A}(G) \in \{ 0,1 \}^{n \times n}$.
\vspace{-0.05in}
\subsection{Message-Passing Neural Networks}
\vspace{-0.05in}
Message-passing neural networks (MPNN) learn $d$-dimensional real-valued vector representations for each vertex in a graph by exchanging and aggregating information from neighboring nodes. Each vertex $v$ is annotated with a feature  $\hb_{v}^\tup{0}$ in $\Rb^{d}$ representing characteristics such as atom positions and numbers in the case of chemical molecules. In addition, each edge $(u, v)$ is associated with a feature vector $\be(u, v)$. %In computational chemistry applications this is typically the bond type. 
An MPNN architecture consists of a composition of permutation-equivariant parameterized functions. 

Following \citet{Gil+2017} and \citet{Sca+2009}, in each layer, $\ell > 0$,  we compute vertex features
\begingroup
\setlength{\abovedisplayskip}{0.1in}
\setlength{\belowdisplayskip}{0.1in}
\begin{align}
	\hb_{v}^\tup{\ell} &\coloneqq
	\UPD^\tup{\ell}\Bigl(\hb_{v}^\tup{\ell-1},\AGG^\tup{\ell} \bigl(\oms \mb_{v,u}^\tup{\ell}
	\mid u\in N(v) \cms \bigr)\Bigr)\nonumber\\
    %%
    % &\in \Rb^{d}, \label{def:gnn} \text{ with}\\
% \end{align}
% with 
% \begin{equation}\label{def:gnn-2}
	\mb_{v,u}^\tup{\ell} &\coloneqq	\bM^\tup{\ell}\Bigl(\hb_{v}^\tup{\ell-1},\hb_{u}^\tup{\ell-1}, \eb_{v,u}\Bigr) \in \Rb^{d},\label{def:gnn-2}
\end{align}
\endgroup
where  $\UPD^\tup{\ell}$,  $\bM^\tup{\ell}$, and $\AGG^\tup{\ell}$ are differentiable parameterized functions. In the case of graph-level regression problems, one uses
\begingroup
\setlength{\abovedisplayskip}{0.05in}
\setlength{\belowdisplayskip}{0.05in}
\begin{equation}\label{def:readout}
	\hb_G \coloneqq \RO\bigl( \oms \hb_{v}^{\tup{L}}\mid v\in V(G) \cms \bigr) \in \Rb^d,
\end{equation}
\endgroup
to compute a single vectorial representation based on learned vertex features after iteration $L$ where $\RO$
can be a differentiable parameterized function.

Molecules are $3$-dimensional structures that can be represented by \textit{geometric graphs}, capturing each atom's 3D position. To obtain more expressive representations, we also consider geometric input attributes and focus on vectorial features $\vv{\bbv}_{v}, \vv{\bbv}_{u}$ of nodes. 
Since we address the problem of molecular property prediction, where we assume the properties to be invariant to actions of the group $E(3)$, we focus on $E(3)$-invariant MPNNs for geometric graphs.

\begin{figure*}
\centering
      \includegraphics[width=.84\textwidth]{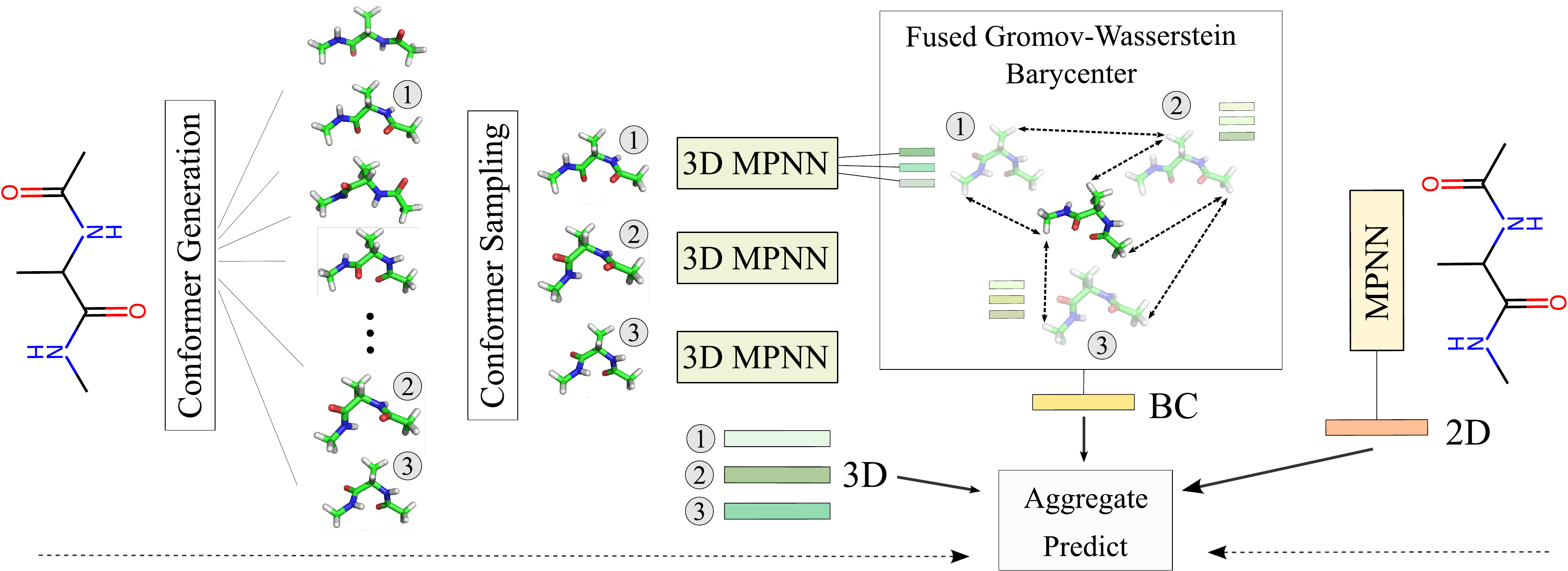}
      \caption{Overview of the proposed conformer aggregation network with alanine dipeptide as example input.}
      \label{fig:trainingOverview}
      \vspace{-0.1in}
\end{figure*}
\vspace{-0.05in}
\subsection{Fused Gromov-Wasserstein Distance}\label{sec_FGW}
%\mn{Here we should add a concise intro to FGW distance and Barycenter computation. Importantly, we need to use terminology that we later, in section \ref{barycenter-aggregation}, instantiate with our 3D invariant features (concretely, from SchNet) and atom distances.}
\vspace{-0.05in}
\textbf{Fused Gromov-Wasserstein.}
An undirected attributed graph $G$ of order $n$  in the optimal transport context is defined as a tuple $G:=(\mH,\mA,\bsomega)$, where $\mH \in \R^{n \times d}$  is
a node feature matrix and $\mA$  is a matrix encoding relationships between nodes, and $\bsomega\in\Delta_n$ denotes the probability measure of nodes within the graph, which can be modeled as the relative
importance weights of graph nodes. 
Without any prior knowledge, uniform weights can be chosen $(\bsomega = \I_n/n)$ \cite{vincent_cuaz_template_2022}. The matrix $\mA$ can be the graph adjacency matrix, the shortest-path matrix or other distance metrics based on the graph topologies \cite{peyre_gromov_wasserstein_2016,titouan_optimal_2019,titouan_fused_2020}.
Given two graphs $G_1, G_2$ of order $n_1, n_2$, respectively, Fused
Gromov-Wasserstein (FGW) distance can be defined as follows:
\setlength{\abovedisplayskip}{5pt}
\setlength{\belowdisplayskip}{5pt}
\allowdisplaybreaks
\begin{align}
  & \hspace{-4mm} \fgwpa(G_1,G_2) \nn\\
    &\hspace{-2mm}  :=  \hspace{-2mm}   \min_{\vpi \in \Pi\left(\bsomega_1, \bsomega_2\right)}\left\langle(1-\alpha) \mM+\alpha \tL\left(\mA_1, \mA_2\right) \otimes \vpi, \vpi\right\rangle. \label{eq_FGW_def}
\end{align}
Here $\mM := \left( d_f(\mH_1[i], \mH_2[j])^p \right)_{n_1 \times n_2} \in \sR^{n_1 \times n_2}$ is the pairwise node distance matrix, $\tL\left(\mA_1, \mA_2\right) = ( |\mA_1[i,j] - \mA_2[l,m]|^p )_{ijlm}$ the 4-tensor representing the alignment cost matrix, and $\bsPi(\bsomega_1,\bsomega_2):=\{\bspi\in\Rp^{n_1\times n_2}|\bspi\I_{n_2} = \bsomega_1,~\bspi\I_{n_1} = \bsomega_2\}$ is the set of all valid couplings between node distributions $\bsomega_1$ and $\bsomega_2$. Moreover, $d_f(\cdot, \cdot)$ is the distance metric in the feature space, and $\alpha \in [0,1]$ is the weight that trades off between the Gromov-Wasserstein cost on the graph structure and Wasserstein cost on the feature signal. In practice, we usually choose $p=2$, Euclidean distance for $d_f(\cdot, \cdot)$, and $\alpha=0.5$ to calculate FGW distance.

\textbf{Entropic Fused Gromov-Wasserstein.} The entropic FGW distance adds an entropic term~\cite{cuturi2013sinkhorn} as
\vspace{-0.05in}
\begingroup
\begin{equation}
\hspace{-4mm} \fgwpa^{\epsilon}(G_1,G_2)  := \fgwpa(G_1,G_2) -  \epsilon \operatorname{H}(\vpi),
\end{equation}
\endgroup
where the entropic scalar $\epsilon$ facilitates the tunable trade-off between solution accuracy and computational performance (w.r.t. lower and higher $\epsilon$, respectively). Solving this entropic FGW involves iterations of solving the linear entropic OT problem~\cref{eq:efgw_update} with (stabilized) Sinkhorn projections (Proposition 2~\cite{peyre2016gromov}), described in~\cref{sec_solving_EFGW} and~\cref{alg:fgw_sinkhorn}. 
\vspace{-0.1in}
\section{\conan: Conformer Aggregation Networks via Fused Gromov-Wasserstein Barycenters}
\vspace{-0.05in}
In what follows, we refer to the representation of atoms and covalent bonds and their attributes as the 2D structure and the atoms, their 3D coordinates, and atom types as 3D structures.  
The following subsections describe each part of the framework in detail.
\vspace{-0.1in}
\subsection{Conformer Generation}
\vspace{-0.05in}
To efficiently generate conformers, we employ distance geometry-based algorithms, which convert distance constraints into Cartesian coordinates. For atomistic systems, constraints typically define lower and upper bounds on interatomic squared distances. In a 2D input graph, covalent bond distances adhere to known ranges, while bond angles are determined by corresponding geminal distances. Adjacent atoms or functional groups adhere to cis/trans limits for rotatable bonds or set values for rigid groups. Other distances have hard sphere lower bounds, usually chosen approximately 10$\%$ below van der Waals radii \cite{hawkins2017conformation}. Chirality constraints are applied to every rigid quadruple of atoms.

A distance geometry algorithm now randomly generates a $3$-dimensional conformation satisfying the constraints. To bias the generation towards low-energy conformations, a simple and efficient force field is typically applied. 
We use efficient implementations from the RDKit package \cite{landrum2016rdkit}.
\vspace{-0.3in}
\subsection{Conformer Aggregation Network}\label{sec_ConAN}
\vspace{-0.05in}
We propose a new MPNN-based neural network that consists of three parts as depicted in Figure \ref{fig:trainingOverview}. 
First, a 2D MPNN model is used to capture the general molecular features such as covalent bond structure and atom features.
Second, a novel FGW barycenter-based implicit $E(3)$-invariant aggregation function that integrates the representations of molecular 3D conformations computed by geometric message-passing neural networks. 
Finally, a permutation and $E(3)$-invariant aggregation function will be used to combine the 2D graph and 3D conformer representations of the molecules. 

\textbf{2D Molecular Graph Message-Passing  Network.}
Each molecule is represented by a 2D graph $G = (V, E)$ with nodes $V$ representing its atoms and edges $E$ representing its covalent bonds, annotated with molecular features $\vh_v^{(0)}$ and $\ve_{v, u}$, respectively (see Section~\ref{seq:expirements} for details).
To propagate features across a molecule and get 2D molecular representations, we use  GAT layers, which utilize a self-attention mechanism in message-passing with the following operations:
\setlength{\abovedisplayskip}{5pt}   
\setlength{\belowdisplayskip}{5pt}
\begin{align}
\label{eq-2d-graph-message-passing}
\hb_{v}^{(\ell)} := & \AGG^\tup{\ell} \bigl(\oms \mb_{v,u}^\tup{\ell} \mid u\in N(v) \cms \bigr) = \sum_{u \in N(v)} \mb_{v,u}^\tup{\ell} \nonumber \\
 & \mbox{ with } \ \ \  \mb_{v,u}^\tup{\ell} = \alpha_{v, u} \Wb \hb_u^{(l-1)},
\end{align}
and where $\alpha_{v, u}$ are the GAT attention coefficients and $\Wb$ a learnable parameter matrix.
Following \citet{Vel+2018}, the attention mechanism is implemented with a single-layer feedforward neural network. To obtain a per-molecule embedding, we compute $    \hb^{\mathtt{2D}}_G = \sum_{v \in V} \hb_v^{(L)}\label{eq-GAT-output}$, where $L$ is the number of message-passing layers. 

\textbf{3D Conformer Message-Passing Network.}
A conformer (atomic structure) of a molecule is defined as $S = \{\bbr_i, Z_i\}_{i=1}^{N}$ where $N$ is the number of atoms, $\bbr_i \in \mathbb{R}^3$ are the Cartesian coordinates of atom $i$, and $Z_i \in \mathbb{N}$ is the atomic number of atom $i$. We use weighted adjacency matrices $\vec{A} \in \mathbb{R}^{n \times n}$ to represent pairwise atom distances. In some cases we will apply a cutoff radius to these distances. 
We employ the geometric MPNN  SchNet~\citep{Sch+2017}, although it is worth noting that alternative $E(3)$-invariant neural networks could be seamlessly integrated. The selection of SchNet is motivated not only by its proficient balance between model complexity and efficacy but also by its proven utility in previous works~\cite{axelrod_molecular_2023}.
SchNet performs $E(3)$-invariant message-passing by using radial basis functions to incorporate the distances of the geometric node features $\vv{\bbv}_{v}, \vv{\bbv}_{u}$. We refer the reader to \cref{appendix-schnet-details} for more details. We denote the matrix whose columns are the atom-wise features of SchNet from the last message-passing layer $L$ with $\Hb$, that is, $\Hb[v] = \hb^{(L)}_v$.

To compute the vector representation for a conformer $S$, we aggregate the atom-wise embeddings obtained from the last message-passing layer $L$ of SchNet into a single vector representation as $\hb_{S}^{\mathtt{3D}} = \sum_{v \in V} \left(\Ab\hb_v^{(L)} + \ba\right)$,  
where $V$ is the set of atoms and $\Ab$ and $\ba$ learned during training. 
For a set of $K$ conformers, the output of our 3D MPNN models is a matrix whose columns are the embeddings $\hb_{S_k}$ for conformer $k$, that is, $\Hb^{\mathtt{3D}}[k] = \hb_{S_{k}}^{\mathtt{3D}}$.

\vspace{0.05in}
\textbf{FGW Barycenter Aggregation.}
\label{barycenter-aggregation}
We now introduce an implicit and differentiable neural aggregation function whose output is determined by solving an FGW barycenter optimization problem. Its input is $K$ graphs $G_k = (\Hb_k, \mA_k,  \bsomega_k)$ for each conformer $S_k = \{\bbr_{k,i}, Z_{k,i}\}_{i=1}^{N}$, with features $\Hb_k$ computed by an $E(3)$-invariant MPNN, with weighted adjacency matrix $\mA_k$ of pairwise atomic distances, and the probability mass of each atom $\bsomega_k$, typically set to $1/N$. The output of the barycenter conformer, denoted as $\overline{G} = (\overline{\Hb}, \overline{\mA}, \overline{\bsomega})$, represents the geometric mean of the input conformers, incorporating both their structural characteristics and features (Figure \ref{fig:trainingOverview}).
The barycenter $\overline{G}$ is the conformer graph that minimizes the sum of weighted FGW distances among the conformer graphs $(G_k)_{k\in[K]}$ with feature matrices $(\Hb_k)_{k\in[K]}$,
structure matrices $(\mA_k)_{k\in[K]}$, and base histograms $(\bsomega_k)_{k\in[K]} \in \Delta_n^K$. That is, given any fixed $K \in \N$ and any $\bslambda \in \Delta_K$, the FGW barycenter is defined as
\begingroup
\setlength{\abovedisplayskip}{-0.05in}
\begin{align}
    \overline{G} :=\argmin_{G} \sum_{k=1}^K \lambda_k \fgwpa(G,G_k),\label{eq_FGWB_def}
\end{align}
\endgroup
where $\fgwpa(G,G_k)$ is the fused Gromov-Wasserstein distance defined in \cref{eq_FGW_def}, and where we set, for each pair of conformer graphs $G = (\Hb, \mA,  \bsomega)$ and $G_k = (\Hb_k, \mA_k,  \bsomega_k)$, $\mM := \left( \left(\Hb[i] - \Hb_k[j]\right)^2\right)_{n \times n} \in \sR^{n \times n}$ as the feature distance matrix, and  $\tL\left(\mA, \mA_k\right) = \left(\mA[i,j] - \mA_k[l, m]|\right)_{ijlm}$ as the 4-tensor representing the structural distance when aligning atoms $i$ to $l$ and $j$ to $m$ (Figure \ref{fig:trainingIllustration}).
Solving~\cref{eq_FGWB_def}, we obtain a unique FGW barycenter graph $\overline{G} = (\overline{\Hb}, \overline{\mA},  \overline{\bsomega})$ with  representation $\overline{\hb}_v = \overline{\Hb}[v]$ for each atom $v$.
We aggregate the atom-wise embeddings obtained from the FGW barycenter $\overline{G}$ into a single vector representation using $    \hb^{\mathtt{BC}}_{\overline{G}} = \sum_{v \in V} \left(\overline{\Ab}~\overline{\hb}_v + \overline{\ba}\right) \label{eq-barycenter-graph-representation}$.

\iffalse
The input embeddings from the upstream 3D-GNN, obtained from the last message-passing layer $L$, are the input to the FGW barycenter aggregation. Note that at this layer $L$, an atom $v$ is represented by a learnable representation $\hb_v^{(L)}$.
\fi

\begin{figure}[t]
\centering    \includegraphics[width=.65\columnwidth]{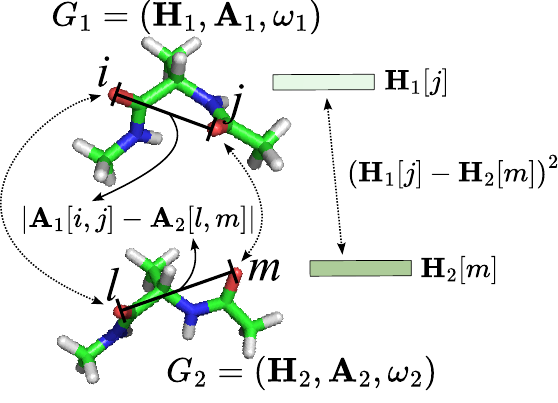}
      \caption{Illustration of the feature-based and structural distances of conformers (here: alanine dipeptide) we use for the computation of the Fused Gromov-Wasserstein barycenter.}
      \label{fig:trainingIllustration}
      \vspace{-0.3in}
\end{figure}

Intuitively, barycenter-based aggregation in Eq.(\ref{eq_FGWB_def}) can be seen as a more distance (structure) preserving pooling operation rather than 
standard mean aggregation. For instance, consider two conformers, where one is a 180-degree rotation of the other. Averaging their coordinates collapses the hydrogen atoms into the same position, creating an unphysical structure. On the contrary, employing the FGW Barycenter might prevent such issues.

\textbf{Invariant Aggregation of 2D and 3D Representations.}
We integrate the representations of the 2D graph and the 3D conformer graphs using an average aggregation as well as the barycenter-based aggregation. The requirement for this aggregation is that it is \emph{invariant} to the order of the input conformers; that is, it treats the conformers as a set as well as invariant to actions of the group $E(3)$. 

Let $\Hb^{\mathtt{2D}}$ and $\Hb^{\mathtt{BC}}$ be the matrices whose columns are, respectively, $K$ copies of the 2D and barycenter representations from
previous sections. Using learnable weight matrices $\Wb^{\mathtt{2D}}$, $\Wb^{\mathtt{3D}}$, and $\Wb^{\mathtt{BC}}$, we obtain the final atom-wise feature matrices as
\vspace{-0.02in}
\begin{equation}
    \Hb^{\mathtt{comb}} = \Wb^{\mathtt{2D}}\Hb^{\mathtt{2D}} + \Wb^{\mathtt{3D}}\Hb^{\mathtt{3D}} + \gamma\, \Wb^{\mathtt{BC}}\Hb^{\mathtt{BC}},
    \label{eq:aggregate}
    \vspace{-0.02in}
\end{equation}
where $\gamma$ is a hyper-parameter controlling the contribution of the barycenter-based feature. Intuitively, this aggregation function, where we use multiple copies of the 2D graph and barycenter representations, provides a balanced contribution of the three types of representations and is empirically highly beneficial. 
Finally, to predict a molecular property, we apply a linear regression layer on a mean-aggregation of the per-conformations embedding as:
\vspace{-0.05in}
\begin{equation}
    \hat{y} = \Wb^{G}\left(\frac{1}{K} \sum_{k=1}^{K} \Hb^{\mathtt{comb}}[k]\right) + \bb^G.
    \label{eq:final_output}
\end{equation}

\vspace{-0.1in}
We can show that the function defined by \cref{eq-2d-graph-message-passing} to \cref{eq:final_output} is invariant to actions of the group $E(3)$ and permutations acting on the sequence of input conformers. 
\begin{theorem}
\label{theorem-invariance}
Let $G$ be the 2D graph and $(S_1, ..., S_K)$ with $S_k = \{\bbr_{k,i}, Z_{k,i}\}_{i=1}^{N}$, $1 \leq k \leq K$, be a sequence of $K$ conformers of a molecule. 
Let $\hat{y} = f_{\bm{\theta}}(G, (S_1, ..., S_K))$ be the function defined by \cref{eq-2d-graph-message-passing} to \cref{eq:final_output}. For any $g_1, ..., g_K \in E(3)$ we have that 
$f_{\bm{\theta}}(G, (g_1  S_1, ..., g_K  S_K)) = f_{\bm{\theta}}(G, (S_1, ..., S_K))$. Moreover, for any $\pi \in \mathrm{Sym}([K])$ we have that $f_{\bm{\theta}}(G, (S_{\pi(1)}, ..., S_{\pi(K)})) = f_{\bm{\theta}}(G, (S_1, ..., S_K))$.
\end{theorem}
\section{Efficient and Convergent Molecular Conformer Aggregation}
\vspace{-0.05in}
In this section, we provide some theoretical results to justify our novel FGW barycenter-based implicit $E(3)$-invariant aggregation function that integrates the representations of molecular 3D conformations computed by geometric message-passing neural networks in \cref{sec_ConAN}.
We established a fast convergence rate of the empirical FGW barycenters to the true barycenters as a function of the number of conformer samples $K$.

\textbf{Undirected Attribute Graph Space.} Let us define a structured object to be a triplet $(\bsOmega,\mA,\mu)$, $\bsOmega=\bsOmega_s\times\bsOmega_f$, where $(\bsOmega_f,d_f)$ and $(\bsOmega_s,\mA)$ are feature and structure metric spaces, respectively, and $\mu$ is a probability measure over $\bsOmega$.
%%
%% Space of structured objects
%%
%%
By defining $\bsomega$, the probability measure of the nodes, the graph $G$ represents a fully supported probability measure over the feature/structure of the product space, $\mu = \sum_{k}\omega_k\delta_{(\vx_k,\va_k)}$, which describes the entire undirected attributed graph.
%%
\iffalse
We also denote $\mu_{\mH_1} = \sum_{k} \omega_k\delta_{\vx_k}$ and $\mu_{\mA_1}=\sum_{k} \omega_k\delta_{\va_k}$ the feature and structure  marginals of $\mu_1$.
\fi
%%
We note $\X$ the set of all metric spaces. The space of all structured objects over $(\bsOmega_f,d_f)$ will be written as $\Sb(\bsOmega)$, and is defined by all the triplets $(\bsOmega,\mA,\mu)$, where $(\bsOmega_f,d_f) \in \X$ and $\mu \in \cP(\bsOmega)$.
%%
\iffalse
To avoid finiteness issues in the rest of the paper, given any $p \in\N$, we only consider the space $\Sb_p(\bsOmega) \subset \Sb(\bsOmega)$ \st
\begin{align}
    &\int_{\bsOmega_f} d(a,b)^p d\mu_{\mA_1}(b) \deq  \sum_{k} h_kd(a,a_k)^p < +\infty \quad \forall a \in \bsOmega_f,\nn\\
    %%
    & \int_{\bsOmega_s \times \bsOmega_s} A(x,y)^p d\mu_{\mH_1}(x) d\mu_{\mH_1}(y)\nn\\
    %%
    & \quad \deq \sum_{i,k} h_i h_kA(x_i,y_k)^p 
    < +\infty.
\end{align}
\fi
\iffalse
The definition of FGW in \eqref{eq_FGW_def} implies that 
\begin{align}
    \int_{\cP(\bsOmega)} \fgwpa^p(\mu,\nu) dP(\nu) < +\infty \quad \forall \mu \in \cP(\bsOmega).
\end{align}

$\Sb_p(\bsOmega) \subset \Sb(\bsOmega)$ to be the set of all Borel probability measures $P$ on $\bsOmega$ \st the variance functional defined by
\begin{align}
    \mu \rightarrow \int_{\bsOmega} d^p(\mu,\nu) dP(\nu) \text{ is finite on $\bsOmega$.}
\end{align}
%%
In our context, it is sufficient to consider $(\bsOmega,d)$ be the FGW space, \ie defined by 
\fi

\textbf{True and Empirical Barycenters.} 
Given $(\bsOmega,\mA,\mu) \in \Sb(\bsOmega)$, the variance functional $\sigma^2$ of a distribution  $P \in \cP(\cP_p(\bsOmega))$ is defined as follows:
\vspace{-0.05in}
\begin{align}
\sigma_P^2 = \int_{\cP_p(\bsOmega)} \fgwpa^p(\overline{\mu}_{0},\nu)dP(\nu),
\vspace{-0.05in}\label{eq_def_sigma2}
\end{align}
where $\overline{\mu}_{0}$ is a \textit{true barycenter} defined in \eqref{eq_true_barycenter}.
We will then restrict our attention to the subset $\cP_p(\cP_p(\bsOmega)) = \left\{P \in \cP(\cP_p(\bsOmega)): \sigma_P^2 < +\infty\right\}$. Note that $\cP_p(\bsOmega)$ is a subset of $\cP(\bsOmega)$ with finite variance and defined the same way as $\cP_p(\cP_p(\bsOmega))$ but on $(\bsOmega,\mA,\mu)$.
For any $P \in \cP_p(\cP_p(\bsOmega))$, we define the true barycenter of $P$ is any $\overline{\mu}_{0} \in \cP_p(\bsOmega)$ \st
\begin{align}
    \overline{\mu}_{0} \in \argmin_{\mu \in \cP_p(\bsOmega)} \int_{\cP_p(\bsOmega)} \fgwpa^p(\mu,\nu) dP(\nu). \label{eq_true_barycenter}
\vspace{-0.15in}
\end{align}
In our context of predicting molecular properties, the true barycenter $\overline{\mu}_{0}$ is unknown. However, we can still draw $K$ random sample independently of the 3D molecular representation  $\left\{\mu_k\right\}_{k \in [K]} = \left\{\sum_{l=1}^k \omega_l\delta_{(x_l,a_l)}\right\}_{k\in[K]}$ from $P$.
Then, an \textit{empirical barycenter} is defined as a barycenter of the empirical distribution $P_K = (1/K)\sum_{k} \delta_{\mu_k}$, \ie
\begin{align}
    \overline{\mu}_K \in \argmin_{\mu \in \cP_p(\bsOmega)} \frac{1}{K}\sum_{k} \fgwpa^p(\mu,\mu_k).
\end{align}
\vspace{-0.25in}
\subsection{Fast Convergence of Empirical FGW Barycenter}\label{sec_FCEB}
\vspace{-0.05in}
This work establishes a novel fast rate convergence for empirical barycenters in the FGW space via \cref{theorem_main}, which is proved in \cref{sec_theorem_main}. To the best of our knowledge, this is new in the literature, where only the result for Wasserstein space exists in  \citet{le_gouic_fast_2022}.
%%
\iffalse
\begin{theorem}\label{theorem_main}
    Let $P \in \cP_2(\cP_2(\bsOmega))$ be a probability measure on the 2-FGW space. Then it holds that:
    \begin{enumerate}
        \item[(a)] There exists at least one barycenter $\overline{\mu}_{0} \in \cP_2(\bsOmega)$ of $P$  satisfying \eqref{eq_true_barycenter}.

        \item[(b)] Let $\gamma,\beta>0$ and suppose that every $\mu \in \supp(P)$ is the pushforward of $\overline{\mu}_{0}$ by the gradient of an $\gamma$-strongly convex and $\beta$ smooth function $\psi_{\overline{\mu}_{0} \rightarrow \mu}$,\ie $\mu = (\nabla \psi_{\overline{\mu}_{0} \rightarrow \mu})_{\#}\overline{\mu}_{0}$. If $\beta-\gamma<1$, then $\overline{\mu}_{0}$ is unique and any empirical barycenter $\overline{\mu}_K$ of $P$ satisfies
    \begin{align}
        \E{}{\fgwtwoa^2(\overline{\mu}_{0},\overline{\mu}_K)} \le \frac{4\sigma^2}{(1-\beta+\gamma)^2K}.
    \end{align}
    \end{enumerate}
    %%
\end{theorem}
\fi

\begin{theorem}\label{theorem_main}
    Let $P \in \cP_2(\cP_2(\bsOmega))$ be a probability measure on the 2-FGW space. Let $\overline{\mu}_{0} \in \cP_2(\bsOmega)$ and $\sigma^2_P$ be barycenter and variance functional of $P$ satisfying (\ref{eq_true_barycenter}) and (\ref{eq_def_sigma2}), respectively. Let $\gamma,\beta>0$ and suppose that every $\mu \in \supp(P)$ is the pushforward of $\overline{\mu}_{0}$ by the gradient of an $\gamma$-strongly convex and $\beta$ smooth function $\psi_{\overline{\mu}_{0} \rightarrow \mu}$, \ie $\mu = (\nabla \psi_{\overline{\mu}_{0} \rightarrow \mu})_{\#}\overline{\mu}_{0}$. If $\beta-\gamma<1$, then $\overline{\mu}_{0}$ is unique and any empirical barycenter $\overline{\mu}_K$ of $P$ satisfies
    \vspace{-0.05in}
    \begin{align}
        \E{}{\fgwtwoa^2(\overline{\mu}_{0},\overline{\mu}_K)} \le \frac{4\sigma^2_P}{(1-\beta+\gamma)^2K}.
        \label{eq:upper_bound}
        \vspace{-0.2in}
    \end{align}
    \iffalse
    This is equivalent to the following
     \begin{align}
        \E{}{\fgwtwoa^2(\overline{G}_{0},\overline{G}_K)} \le \frac{4\sigma^2_P}{(1-\beta+\gamma)^2K}. \label{eq:upper_bound}
    \end{align}
    \fi
    %%
\end{theorem}
\vspace{-0.1in}
The upper bound in \cref{eq:upper_bound} implies that the empirical barycenter converges to the target distribution at a rate of $\mathcal{O}(1/K)$, where $K$ is the number of 3D conformers. This suggests utilizing small values of $K$, such as $K \in \{5, 10\}$, would yield a satisfactory approximation for $\overline{\mu}_{0}$. We confirm this empirically in experiments in Section \ref{sec:ablation_study}.
\vspace{-0.05in}
\begin{algorithm}[H]
   \caption{Entropic FGW Barycenter}
   \label{alg:fgw_barycenter}
   \scalebox{0.9}{
   \begin{minipage}{\linewidth}
\begin{algorithmic}
   \STATE {\bfseries Input:} $\overline{\bsomega}$, $\{G_s := (\mH_s, \mA_s,\bsomega_s)\}_{s=1}^K$, $\epsilon$.
   \STATE {\bfseries Optimizing:} $\overline{G}, \{\vpi_s \in \Pi(\overline{\bsomega}, \bsomega_s)\}_{s=1}^K$.
   \REPEAT
   %\STATE Initialize $\overline{G}^{(k)} = (\overline{\bsomega}, \overline{\mH}^{(k)}, \overline{\mA}^{(k)})$.
   \FOR{$s=1$ {\bfseries to} $K$}
   \STATE Solve $\argmin_{\vpi_s^{(k)}} \fgwpa^{\epsilon}(\overline{G}^{(k)}, G_s)$ with Alg.~\ref{alg:fgw_sinkhorn}.
   \ENDFOR
   \STATE Update $\overline{\mA}^{(k+1)} \leftarrow \frac{1}{\overline{\bsomega}\overline{\bsomega}^\top} \frac{1}{K} \sum_{s=1}^K {\vpi_s^{(k)}} \mA_s {\vpi_s^{(k)}}^\top$.
   \STATE Update $\overline{\mH}^{(k+1)} \leftarrow \mathrm{diag}(1/\overline{\bsomega}) \frac{1}{K} \sum_{s=1}^K {\vpi_s^{(k)}} \mH_s$
   \UNTIL{$k$ in \textit{outer iterations} and \textit{not converged}}
\end{algorithmic}
\end{minipage}}
\end{algorithm}
\vspace{-0.2in}
\subsection{Empirical Entropic FGW Barycenter}
To train on large-scale problems, we propose to solve the entropic relaxation of~\cref{eq_FGWB_def} to take advantage of GPU computing power \cite{peyre2019computational}.
% for utilizing GPU-accelerated Sinkhorn iterations~\cite{peyre2019computational}. 
Given a set of conformer graphs $\{G_s := (\mH_s, \mA_s,\bsomega_s)\}_{s=1}^K$, we want to optimize the entropic barycenter $\overline{G}$, where we fixed the prior on nodes $\overline{\bsomega}$
\begingroup
\setlength{\abovedisplayskip}{-0.05in}
\setlength{\belowdisplayskip}{0pt}
\begin{equation} \label{eq:fgw_barycenter}
\overline{G} = \argmin_{G} \frac{1}{K} \sum_{s=1}^K \fgwpa^{\epsilon}\left(\overline{G}, G_s\right).
\end{equation}
\endgroup

with $\lambda_s = 1/K,\,\forall s \in [1, K]$. \citet{titouan_optimal_2019} solve~\cref{eq:fgw_barycenter} using Block Coordinate Descent, which iteratively minimizes the original FGW distance between the current barycenter and the graphs $G_s$. In our case, we solve for $K$ couplings of entropic FGW distances to the graphs at each iteration, then following the update rule for structure matrix (Proposition 4,~\cite{peyre2016gromov})
\begingroup
\setlength{\abovedisplayskip}{0.02pt}
\setlength{\belowdisplayskip}{0.02pt}
\begin{equation}
    \overline{\mA}^{(k+1)} \leftarrow \frac{1}{\overline{\bsomega}~\overline{\bsomega}^\top} \frac{1}{K}\sum_{s=1}^K {\vpi_s^{(k)}} \mA_s {\vpi_s^{(k)}}^\top,
\end{equation}
\endgroup
and for the feature matrix~\cite{titouan_optimal_2019, cuturi2014fast}
\vspace{-0.05in}
\begingroup
\setlength{\abovedisplayskip}{0.04pt}
\setlength{\belowdisplayskip}{0.04pt}
\begin{equation}
    \overline{\mH}^{(k+1)} \leftarrow \mathrm{diag}(1/\overline{\bsomega}) \frac{1}{K} \sum_{s=1}^K {\vpi_s^{(k)}} \mH_s,
\end{equation}
\endgroup
leading to~\cref{alg:fgw_barycenter}. More details on practical implementations and algorithm complexity are in~\cref{sec_solving_EFGW}.
\vspace{-0.25in}
\section{Related Work}
\vspace{-0.05in}
\textbf{Molecular Representation Learning.}
 The traditional approach for molecular representation referred to as connectivity fingerprints \cite{Morgan1965} encodes the presence of different substructures within a compound in the form of a binary vector. Modern molecular representations used in machine learning for molecular properties prediction include 1D strings \cite{ChemBerta2,SmilesBert}, 2D topological graphs \cite{Gil+2017,Yang2019,rong2020self,hu2020pretraining} and 3D geometric graphs \cite{ChemRL-GEM,zhou2023unimol,liu2022pretraining}. The use of an ensemble of molecular conformations remains a relatively unexplored frontier in research, despite early evidence suggesting its efficacy in property prediction \cite{axelrod_molecular_2023, wang2024multi}.
 Another line of work uses conformers only at training time in a self-supervised loss to improve a 2D MPNN~\cite{pmlr-v162-stark22a}. Contrary to prior work, we introduce a novel and streamlined barycenter-based conformer aggregation technique, seamlessly integrating learned representations from both 2D and 3D MPNNs. 
 Moreover, we show that cost-effective conformers generated through distance-geometry sampling are sufficiently informative.

\textbf{Geometric Graph Neural Networks.}
 Graph Neural Networks (GNNs) designed for geometric graphs operate based on the message-passing framework, where the features of each node are dynamically updated through a process that respects permutation equivariance. Examples are models such as SphereNet \cite{SphereNet}, GMNNs \cite{zaverkin:2020}, DimeNet \cite{gasteiger_dimenet_2020}, GemNet-T \cite{gasteiger2021gemnet}, SchNet \cite{schutt2017schnet}, GVP-GNN, PaiNN, E(n)-GNN \cite{satorras2021En}, MACE \cite{batatia2022mace}, ICTP \cite{zaverkin2024higherrank}, SEGNN \cite{brandstetter2022geometric}, SE(3)-Transformer \cite{fuchs2020se}, and VisNet \cite{wang2024enhancing}.

\textbf{Optimal Transport in Graph Learning.}
By modeling graph features/structures as probability measures, the (Fused) GW distance \cite{titouan_fused_2020} serves as a versatile metric for comparing structured graphs. Previous applications of GW distance include graph node matching \cite{xu2019gromov}, partitioning \cite{xu2019scalable,chowdhury2021generalized}, and its use as a loss function for graph metric learning \cite{vincent2021online, vincent_cuaz_template_2022, chen2020optimal, zeng2023generative}. More recently, FGW has been leveraged as an objective for encoding graphs \cite{tang2023fused} in tasks such as graph prediction \cite{brogat_learning_2022} and classification \cite{ma2023fused}.
To the best of our knowledge, we are the first to introduce the entropic FGW barycenter problem~\cite{peyre_gromov_wasserstein_2016, titouan_fused_2020} 
for molecular representation learning. By employing the entropic formulation~\cite{cuturi2013sinkhorn, cuturi2014fast}, our learning pipeline enjoys a tunable trade-off between barycenter accuracy and computational performance, thus enabling an efficient hyperparameter tuning process. Moreover, we also present empirical barycenter-related theories,  demonstrating how this entropic FGW barycenter framework effectively captures meaningful underlying structures of 3D conformers, thereby enhancing overall performance.
\vspace{-0.1in}
\section{Experiments} 
\label{seq:expirements}
\vspace{-0.05in}
\subsection{Implementation Details}
\vspace{-0.05in}
We encode each molecule in the SMILES format and employ the RDKit package to generate 3D conformers. We set the size of the latent dimensions of GAT~\cite{Vel+2018} to $128/256$. Node features are initialized based on atomic properties such as atomic number, chirality, degree, charge, number of hydrogens, radical electrons, hybridization, aromaticity, and ring membership, while edges are represented as one-hot vectors denoting bond type, stereo configuration, and conjugation status.
Each 3D conformer generated by RDKit comprises $n$ atoms with the corresponding 3D coordinates representing their spatial positions. Subsequently, we establish the graph structure and compute atomic embeddings utilizing the force-field energy-based SchNet model \cite{Sch+2017}, extracting features prior to the $\RO$ layer. Our SchNet configuration incorporates \textit{three interaction blocks} with feature maps of size $F = 128$, employing a radial function defined on Gaussians spaced at intervals of $0.1 \mathring{{A}}$ with a cutoff distance of 10 $\mathring{A}$. The output of each conformer $k \in [K]$ forms a graph $G_k$, utilized in solving the FGW barycenter $\overline{G}$ as defined in Eq. (\ref{eq_FGWB_def}). Subsequently, we aggregate features from 2D, 3D, and barycenter molecule graphs using Eqs. (\ref{eq:aggregate}-\ref{eq:final_output}), followed by MLP layers. Leveraging Sinkhorn iterations in our barycenter solver (Algorithm \ref{alg:fgw_barycenter}), we speed up the training process across multiple GPUs using PyTorch's distributed data-parallel technique. Training the entire model employs the Adam optimizer with initial learning rates selected from ${1e^{-3}, 1e^{-3}/2, 1e^{-4}}$, halved using ReduceLROnPlateau after $10$ epochs without validation set improvement. Further experimental details are provided in the Appendix.

To accelerate the training process, especially in large-scale settings (e.g., BDE dataset), we first train the model with 2D and 3D features for some epochs, and then load the saved model and continue to train with full configurations as in Eq.(\ref{eq:aggregate}) till converge. We set empirically $\gamma$ in Eq.(\ref{eq:aggregate}) is $0.2$. 
\vspace{-0.15in}
\begin{table}[H]
\caption{Number of samples for each split on molecular property prediction, classification tasks, and reaction prediction.}
\setlength{\tabcolsep}{2.8pt}
\vspace{0.05in}
\centering
\scalebox{0.8}{
\begin{tabular}{lccccccc}
\hline
& \textbf{Lipo} & \textbf{ESOL} & \textbf{FreeSolv} &\textbf{BACE} & \textbf{CoV-2 3CL}  & \textbf{Cov-2} & \textbf{BDE}\\
\hline
\textbf{Train} & 2940 & 789 & 449 & 1059 & 50 (485) & 53 (3294) & 8280\\
\textbf{Valid.} & 420 & 112 & 64 & 151  & 15 (157) & 17 (1096) &  1184 \\
\textbf{Test} &  840 & 227 & 129 & 303 & 11 (162) & 22 (1086)  & 2366 \\
\hline
\textbf{Total} &  4200 & 1128 & 642 & 1513 &  76 (804) & 92 (5476) & 11830 \\
\hline
\end{tabular}}
\label{tab:regression-classification}
\end{table}
\vspace{-0.2in}
\subsection{Molecular Property Prediction Tasks}
\vspace{-0.05in}
\textbf{Dataset.} We use four datasets \texttt{Lipo}, \texttt{ESOL}, \texttt{FreeSolv}, and \texttt{BACE} in \texttt{MoleculeNet} benchmark (Table \ref{tab:regression-classification}), spanning on various molecular characteristics such as physical chemistry and biophysics. We split data using random scaffold settings as baselines and reported the mean and standard deviation of mean square error (mse) by running on five trial times. More information for datasets is in Section \ref{sec:dataset_overview} Appendix.

\textbf{Baselines.} We compare against various benchmarks, including both supervised, pre-training, and multi-modal approaches.
\texttt{(i) The supervised methods} are 2D graph neural network models including 2D-GAT \cite{Vel+2018}, D-MPNN \cite{Yang2019}, and AttentiveFP \cite{xiong2019pushing}; \texttt{(ii) 2D molecule pretraining methods} are PretrainGNN \cite{Hu2020Strategies}, GROVER \cite{rong2020self}, MolCLR \cite{wang2022molecular}, ChemRL-Gem \cite{fang2022geometry}, ChemBERTa-2 \cite{ChemBerta2}, and MolFormer \cite{ross2022large}. It's important to note that these models are pre-trained on a vast amount of data; for example, MolFormer is learned on  $1.1$ billion molecules from PubChem and ZINC datasets. 
We also compare with the \texttt{(iii) 2D-3D aggregation} ConfNet model \cite{liu2021fast}, which is one of the winners of KDD Cup on OGB Large-Scale Challenge \cite{hu2021ogblsc}.
Finally, we benchmark with \texttt{3D conformers-based models} such as UniMol \cite{zhou2023unimol}, SchNet, and ChemProp3D \cite{axelrod_molecular_2023}. Among this, UniMol is pre-trained on $209$ M molecular conformation and requires 11 conformers on each downstream task. We train SchNet with $5$ conformers ($10$ for FreeSolv) and test with two versions: (a) taking output at the final layer and averaging different conformers (SchNet-scalar), (b) using feature node embeddings before $\RO$ layers and aggregating conformers by an MLP layer (SchNet-em). In ChemProp3D, we replace the classification header with an MLP layer for regression tasks, training with a 2D molecular graph and $10$ conformers. With the ConfNet, we use $20$ conformers in the training step and provide results for $20 \textrm{ and } 40$ conformers for the evaluations step, followed by configurations in \cite{liu2021fast}.

\vspace{-0.2in}
\begin{table}[H]
\centering
\caption{Models evaluation on regression tasks (MSE $\downarrow$).}
\vspace{0.05in}
\label{tab:regression_task}
\setlength{\tabcolsep}{2.pt}
\resizebox{1.02\columnwidth}{!}{%
\begin{tabular}{lllll}
    \toprule
    \multicolumn{1}{c}{\bf Model}  &\multicolumn{1}{c}{\bf Lipo} &\multicolumn{1}{c}{\bf ESOL} &\multicolumn{1}{c}{\bf FreeSolv} &\multicolumn{1}{c}{\bf BACE} \\
    \cline{1-5}
    2D-GAT & $1.387 \pm 0.206$ & $2.288 \pm 0.017$ & $8.564 \pm 1.345$ & $1.844 \pm 0.33$ \\
    D-MPNN & $0.534 \pm 0.022$ & $0.923 \pm 0.045$ & $4.213 \pm 0.068$  & $0.723 \pm 0.021$ \\
    Attentive FP & $	0.520 \pm 0.001$ & $0.771 \pm 0.026$ & $4.197 \pm 0.193$ & - \\
    PretrainGNN & $	0.545 \pm 0.003$ & $	1.21 \pm 0.005$ & $6.392 \pm 0.003$ & - \\
    GROVER\_{large} & $	0.676 \pm 0.012$ & $	0.798 \pm 0.018$ & $5.162 \pm 0.047$ & - \\
    ChemBERTa-2* & $ 0.639 \pm 0.006$ & $	0.795 \pm 0.033$ & - & $1.858 \pm 0.029$ \\
    ChemRL-GEM & $0.486 \pm 0.008$ & $0.706 \pm 0.061$ & $3.924 \pm 0.436$  & - \\
    MolFormer & $0.492 \pm 0.012$ & $0.766 \pm 0.026$ & $5.485 \pm 0.045$ & $1.091 \pm 0.021$ \\
    % SPMM & $0.692 \pm 0.008$ & $0.818 \pm 0.008$ & $1.907 \pm 0.058$  & $1.096 \pm 0.011$ \\
    % MolCLR & $0.691 \pm 0.004$ & $1.271 \pm 0.040$ & $2.594 \pm 0.249$ & \\
    UniMol & \boldsymbol{$0.374 \pm 0.012$} & $0.741 \pm 0.014$ & {$2.867 \pm 0.186$}& - \\
    SchNet-scalar & $0.704 \pm 0.032$ & $0.672 \pm 0.027$  & $1.608 \pm 0.158$ & $0.723 \pm  0.1$ \\
    SchNet-emb & $0.589 \pm 0.022$ & $0.635 \pm 0.057$ & $1.587 \pm 0.136$ & {$0.692 \pm 0.028$} \\
    ChemProp3D & $0.602 \pm 0.035$ & $0.681 \pm 0.023$ & $2.014 \pm 0.182$ & $0.815 \pm 0.17$ \\
    ConfNet & $1.360 \pm 0.038$ & $2.115 \pm 0.484$ & - & $1.329 \pm 0.042$ \\
    \midrule
    \conan & {$0.556 \pm 0.013$} & \underline{$0.571 \pm 0.019$} & $1.496 \pm 0.158$ & \underline{$0.635 \pm 0.051$} \\
    \conan-FGW & \underline{$0.422 \pm 0.016$} & \boldsymbol{$0.529 \pm 0.022$} & \boldsymbol{$ 1.068 \pm  0.083 $} & \boldsymbol{$0.549 \pm 0.016$} \\
    \bottomrule
\end{tabular}
}%
\end{table}
\vspace{-0.2in}
\textbf{Results.} 
Table \ref{tab:regression_task} presents the experimental findings of \conan, alongside competitive methods, with the best results highlighted in bold. Baseline outcomes from prior studies \cite{zhou2023unimol,fang2022geometry,chang2023bidirectional} are included, while performance for other models is provided through public codes. \conan \, version denotes the aggregation of 2D and 3D features as per Eq. (\ref{eq:aggregate}) without employing the barycenter, whereas \conan-FGW signifies full configurations. We employ a number of conformers $\{5, 5, 10, 5\}$ and $\{5, 5, 5, 5\}$ for \conan , and \conan-FGW, respectively, based on validation results for \texttt{Lipo}, \texttt{ESOL}, \texttt{FreeSolv}, and \texttt{BACE}. From the experiments, several observations emerge: (i) \conan \, proves more effective than relying solely on 2D or 3D, as shown by Conan's performance, achieving second-best rankings on three datasets compared to models using only 2D (ChemRL-GEM) or 3D representations (UniMol). (ii) \conan-FGW consistently outperforms baselines across all datasets, despite employing significantly fewer 3D conformers than \conan. This underscores the importance of leveraging the barycenter to capture invariant 3D geometric characteristics.
\vspace{-0.05in}
\subsection{3D SARS-CoV Molecular Classification Tasks}
\vspace{-0.05in}
\textbf{Dataset.} \ We evaluate \conan \,on two datasets \texttt{Cov-2 3CL} and \texttt{Cov-2} (Table \ref{tab:regression-classification}), focusing on molecular classification tasks. The same splitting for training and testing is followed \cite{axelrod_molecular_2023}. We also apply the CREST \cite{grimme2019exploration} to filter generated conformers by RDKit as \cite{axelrod_molecular_2023} for fair comparisons. Model performance is reported with the receiver operating characteristic area under the curve (ROC) and precision-recall area under the curve (PRC) over three trial times.

\textbf{Baselines.} We compare with three models, namely, SchNetFeatures, ChemProp3D, CP3D-NDU, each with two different attention mechanisms to \textit{ensemble 3D conformers and 2D molecular graph} feature embedding as proposed by \citet{axelrod_molecular_2023}. These baselines generate $200$ conformers for their input algorithms. Additionally, the ConfNet \cite{liu2021fast} is also evaluated using $20$ or $40$ conformers in testing. 
\vspace{-0.2in}
\begin{table}[H]
\caption{Performance of various models on the two molecular classification tasks.}
\vspace{0.1in}
\label{tab:cls-task}
\centering
\setlength{\tabcolsep}{2.5pt}
\resizebox{1.0\columnwidth}{!}{%
\begin{tabular}{ccccc}
\toprule
\textbf{Method} & \textbf{Num Conformers} & \textbf{Dataset} & \textbf{ROC} $\uparrow$ & \textbf{PRC} $\uparrow$ \\
\midrule
SchNetFeatures & 200 & CoV-2 3CL & 0.86 & 0.26 \\ 
ChemProp3D & 200 & CoV-2 3CL & 0.66 & 0.20  \\ 
CP3D-NDU & 200 & CoV-2 3CL & 0.901 & 0.413 \\ 

SchNetFeatures & average neighbors & CoV-2 3CL & 0.84 & 0.29 \\ 
ChemProp3D & average neighbors & CoV-2 3CL & 0.73 & 0.31 \\ 
CP3D-NDU & average neighbors & CoV-2 3CL & \underline{0.916} & \textbf{0.467} \\ 
ConfNet & $\{20, 40\}$ & CoV-2 3CL & 0.493  & 0.128 \\
$\conan$ & 10 & CoV-2 3CL & 0.881 $\pm$ 0.009 & 0.317 $\pm$ 0.052 \\ 
$\conan$-FGW & 5 & CoV-2 3CL & \textbf{0.918 $\pm$ 0.012} & \underline{0.423 $\pm$ 0.045} \\
\midrule
SchNetFeatures & 200 & CoV-2 & 0.63 & 0.037 \\ 
ChemProp3D & 200 & CoV-2 & 0.53 & 0.032 \\ 
CP3D-NDU & 200 & CoV-2 & \underline{0.663} & \underline{0.06} \\ 
SchNetFeatures & average neighbors & CoV-2 & 0.61 & 0.027 \\ 
ChemProp3D & average neighbors & CoV-2 & 0.56 & \textbf{0.10} \\ 
CP3D-NDU & average neighbors & CoV-2 & 0.647 & 0.058 \\ 
ConfNet & $\{20, 40\}$ & CoV-2  & 0.501 $\pm$ 0.001  & 0.36 $\pm$ 0.2 \\
$\conan$ & 10 & CoV-2 & 0.634 $\pm$ 0.053 & 0.031 $\pm$ 0.023 \\ 
$\conan$-FGW & 10 & CoV-2 & \textbf{0.6735 $\pm$ 0.032} & 0.036 $\pm$ 0.014 \\ 
\bottomrule
\end{tabular}
}
\end{table}
\vspace{-0.2in}

\textbf{Results.} Table \ref{tab:cls-task} presents performance of \conan \, and \conan-FGW with the number of conformers $10 \text{ or }5$. It can be seen that \conan-FGW delivers the best performance on ROC metric on two datasets and holds the second-best rank with PRC on CoV-2-3CL while requiring only $10 \text{ or } 5$ conformers compared with $200$ conformers as CP3D-NDU. These results underscore the efficacy of incorporating barycenter components over merely aggregating 2D and 3D conformer embeddings, as observed in \conan.
\subsection{Molecular Conformer Ensemble Benchmark}
\vspace{-0.05in}
\textbf{Dataset.} We run \conan on the BDE dataset (Table \ref{tab:regression-classification}), which is the second-largest setting in \citep{zhu2023learning} aim to predict reaction-level molecule properties. \\
\textbf{Baselines.} \conan\,is compared with state-of-the-art  conformer ensemble strategies presented in \citet{zhu2023learning}, including SchNet \cite{Sch+2017}, DimeNet++ \cite{gasteiger2020directional}, GemNet \cite{gasteiger2021gemnet}, PaiNN \cite{Schuett2021}, ClofNet \cite{du2022se}, and LEFTNet \cite{du2024new}. All these approaches employ $20$ conformers in training and testing. We provide two results of \conan\, \textit{using only $10$ conformers} and based on two architectures, SchNet and LEFTNet.\\
\textbf{Results.} Table \ref{tab:conformer-ensemble} summarizes our achieved scores where the \conan-FGW using LEFTNet backbone holds the second rank overall while using half the number of conformers. Additionally, it can be seen that \conan-FGW improves with significant margins over both base models like SchNet ($1.9737 \rightarrow 1.6047$) and LEFTNet ($1.5276 \rightarrow 1.4829$), demonstrating the generalization of the proposed aggregation.
\vspace{-0.2in}
\begin{table}[H]
\caption{Performance of different conformer ensemble strategies on reaction molecules prediction. Results are in Mean Absolute Error (MAE $\downarrow$). \conan-FGW$^{1}$ and \conan-FGW$^{2}$ denote for our versions using SchNet and LEFTNet, respectively.}
\setlength{\tabcolsep}{2.5pt}
\vspace{0.05in}
\centering
\scalebox{0.56}{
\begin{tabular}{lccccccccc}
\toprule
 & SchNet & {DimeNet++} & {GemNet} & {PainNN} & {ClofNet}  & {LEFTNet} & \small{{\conan-FGW$^{1}$}} & \small{{\conan-FGW$^{2}$}}\\
\hline
\vspace{0.05in}
\textbf{Conf.} &20  & 20 & 20 & 20 & 20 & 20 & 10 & 10 \\
\textbf{MAE $\downarrow$} & 1.9737 & \textbf{1.4741} & 1.6059 & 1.8744 &  2.0106 & 1.5276 & \textit{1.6047} & \textbf{\underline{1.4829}} \\
\bottomrule
\end{tabular}}
\label{tab:conformer-ensemble}
\end{table}
\vspace{-0.25in}
\subsection{Ablation Study}
\label{sec:ablation_study}
\vspace{-0.05in}
\textbf{Contribution of 3D Conformer Number.}\ One of the building blocks of our model is the use of multiple 3D conformations of a molecule. Each molecule is represented by $K$ conformations, so the choice of $K$ affects the model's behavior. We treat $K$ as a hyperparameter and conduct experiments to  
validate the impact on model performance. To this end, we test on the \conan \, version with different $K$ ($K=0$ is equivalent to the 2D-GAT baseline) and report performance in Table \ref{tab:k-choice} Appendix.
We can observe that using 3D conformers with $K \geq 1$ clearly improves performance compared to using only 2D molecular graphs as 2D-GAT. Furthermore, there is no straightforward dependency between the number of conformations in use and the accuracy of the model. For e.g., 
the performance tends to increase when using $K=10$ (Lipo and FreeSolv), but overall, the best trade-off value is $K=5$.
\vspace{-0.1in}
\begin{figure}[H]
\centering   \includegraphics[width=0.78\columnwidth]{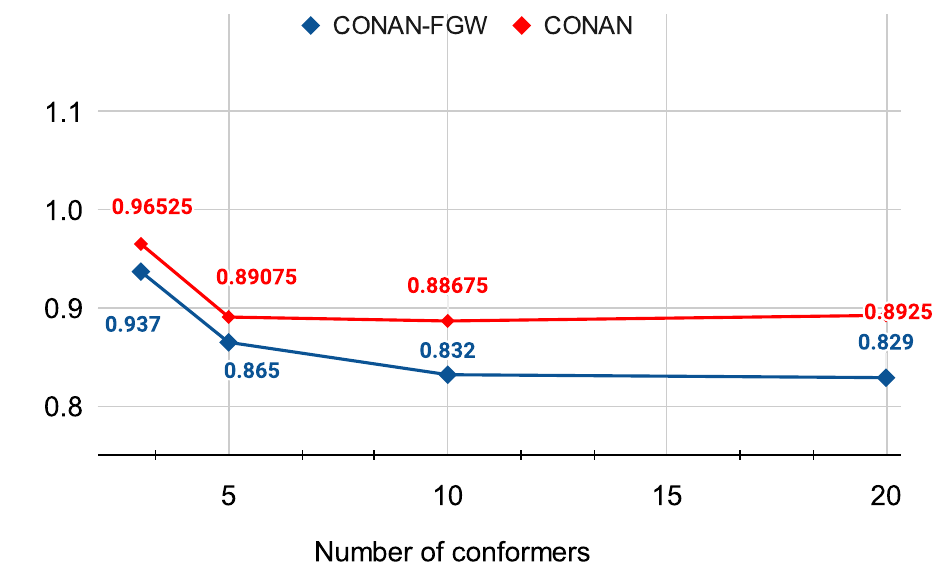}
\vspace{-0.1in}
      \caption{Ablation study on the effect of number conformers on the FGW barycenter component on valid sets.}
      \label{fig:ablation_fgw}
\end{figure}
\vspace{-0.2in}
\textbf{Contribution of FGW Barycenter Aggregation.} \ We examine the effect of barycenter aggregation 
when varying the number of conformers $K$.
Figure \ref{fig:ablation_fgw} summarizes results for those settings where we report average RMSE over four datasets in the MoleculeNet benchmark. We draw the following observations. First, \conan-FGW shows notable enhancements as the number of conformers increases, with $K$ values ranging within the set ${3, 5, 10}$;
however, when as $K=20$, discernible disparities compared to the results obtained at $K=10$ diminish. We argue that this phenomenon aligns consistently with theoretical results in \textbf{Theorem \ref{theorem_main}} suggesting that employing a sufficiently large $K$ facilitates a precise approximation of the target barycenter.

Secondly, upon examining various datasets, it becomes evident that \conan-FGW consistently demonstrates enhanced performance with the utilization of larger conformers, a phenomenon not uniformly observed in the case of \conan. This observation validates the robustness and resilience inherent in \conan-FGW. We attribute this advantage to the efficacy of its geometry-informed aggregation strategy in ensemble learning with diverse 3D conformers.

\textbf{Generalization to other Backbone Model.}
We investigate \conan\, and \conan-FGW performance using the VisNet backbone \cite{wang2024enhancing}, an equivariant geometry-enhanced graph neural network for 3D conformer embedding extraction. Results in Table \ref{tab:visnet} confirm that \conan-FGW still advances \conan\, performance. Between VisNet and SchNet, there is no universal best choice over datasets.
\vspace{-0.2in}
\begin{table}[H]
\centering
\caption{\conan\, evaluation using VisNet and SchNet on regression tasks (MSE $\downarrow$).}
\vspace{0.1in}
\label{tab:visnet}
\setlength{\tabcolsep}{2.5pt}
\resizebox{0.9\columnwidth}{!}{%
\begin{tabular}{lllll}
    \toprule
    \multicolumn{1}{c}{\bf Model}  &\multicolumn{1}{c}{\bf Lipo} &\multicolumn{1}{c}{\bf ESOL} &\multicolumn{1}{c}{\bf FreeSolv} &\multicolumn{1}{c}{\bf BACE} \\
    \cline{1-5}
    \conan\, (VisNet) & {$0.554 \pm 0.448$} & {$1.025 \pm 0.119$} & $0.692 \pm 0.032$ & $0.612 \pm 0.148$ \\
    \conan-FGW & {$0.495 \pm 0.008$} & {$0.552 \pm 0.052$} & \boldsymbol{$ 0.643 \pm  0.015 $} & \boldsymbol{$0.469 \pm 0.012$} \\
    \midrule
    \conan\,(SchNet) & {$0.556 \pm 0.013$} & {$0.571 \pm 0.019$} & $1.496 \pm 0.158$ & $0.635 \pm 0.051$ \\
    \conan-FGW & \boldsymbol{$0.422 \pm 0.016$} & \boldsymbol{$0.529 \pm 0.022$} & {$ 1.068 \pm  0.083 $} & {$0.549 \pm 0.016$} \\
    \bottomrule
\end{tabular}
}%
\end{table}
\vspace{-0.3in}
\subsection{3D Conformers distance distribution}
\vspace{-0.05in}
We check diversity conformers randomly selected from a set of conformers generated by RDKit. For each pair of 3D conformers, we compute the optimal root mean square distance, which first aligns two molecules before measuring distance. Two settings are conducted: (i) estimating the mean, variance, max, and min distances distribution for conformers sampled by \conan\, over $200$ conformers generated by RDKit. and (ii) estimate distribution for those values in case they are the top closest conformers. Figure \ref{fig:visualize_distance} below shows our observation with a box plot on the validation set of Fressolv.
\vspace{-0.1in}
\begin{figure}[H]
\centering
\includegraphics[width=0.85\columnwidth]{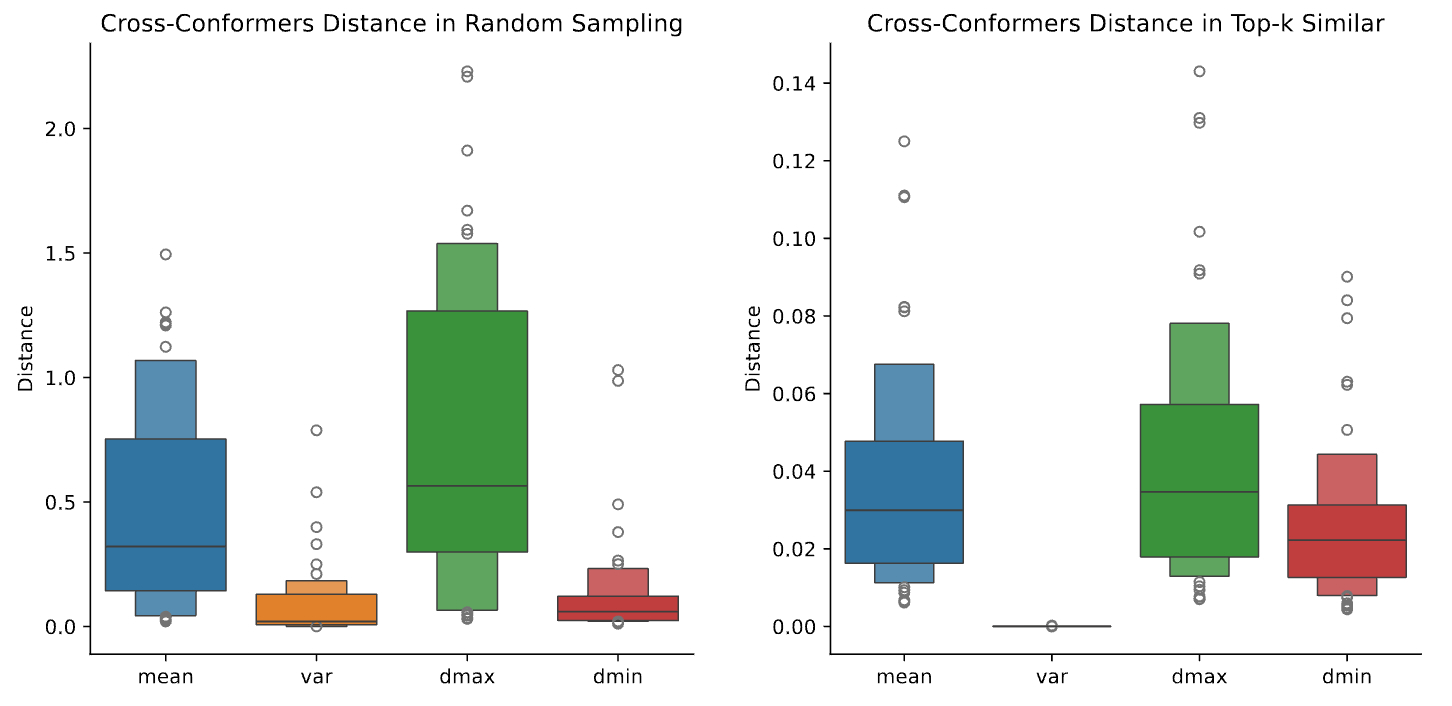}
\vspace{-0.15in}
      \caption{(left) box-plot distribution of mean, variance, maximum, and minimum distances among conformers; (right) distribution of the same values where sample top-$k$ closest conformers.}
      \label{fig:visualize_distance}
\end{figure}
\vspace{-0.2in}
We observe that the distribution on the left ranges from 0.1 to 1.5, while in the worst case, the distance is between 0.01 and 0.08. Additionally, there's a large gap between $d_{\mathrm{max}}$ and $d_\mathrm{min}$ on the left, whereas on the right, their means are close. It, therefore, can be seen that \conan \, sampling strategy, given 200 RDKit-generated conformers, remains consistent and diverse.

\vspace{-0.1in}
\subsection{FGW Barycenter Algorithm Efficiency} \label{subsection-efficiency}
\vspace{-0.05in}
We contrast our entropic solver (Algorithm \ref{alg:fgw_barycenter}) with FGW-Mixup \cite{ma2023fused}
for the $K$ barycenter problem. FGW-Mixup accelerates FGW problem-solving by relaxing coupling feasibility constraints. However, as the number of conformers $K$ increases, FGW-Mixup requires more outer iterations due to compounding marginal errors in solving $K$ FGW distances. In contrast, our approach employs an entropic-relaxation FGW formulation ensuring that marginal constraints are respected, resulting in a less noisy FGW subgradient. Furthermore, we implement our algorithm with distributed computation on multi-GPUs, as highlighted in Fig. \ref{fig:runtime_freesolv_cov23cl}. This figure illustrates epoch durations during both forward and backward steps of training, showcasing the performance across various conformer setups on FreeSolv and CoV-2 3CL datasets. Utilizing a batch size of 32 conformers, all three algorithms employ early termination upon reaching error tolerance. Notably, our solver exhibits linear scalability with \(K\), while FGW-Mixup shows exponential growth, presenting challenges for large-scale learning tasks. More details are in~\cref{fgwmixup-details}.

\vspace{-0.15in}
\begin{figure}[H]
\centering
\includegraphics[width=0.95\columnwidth]{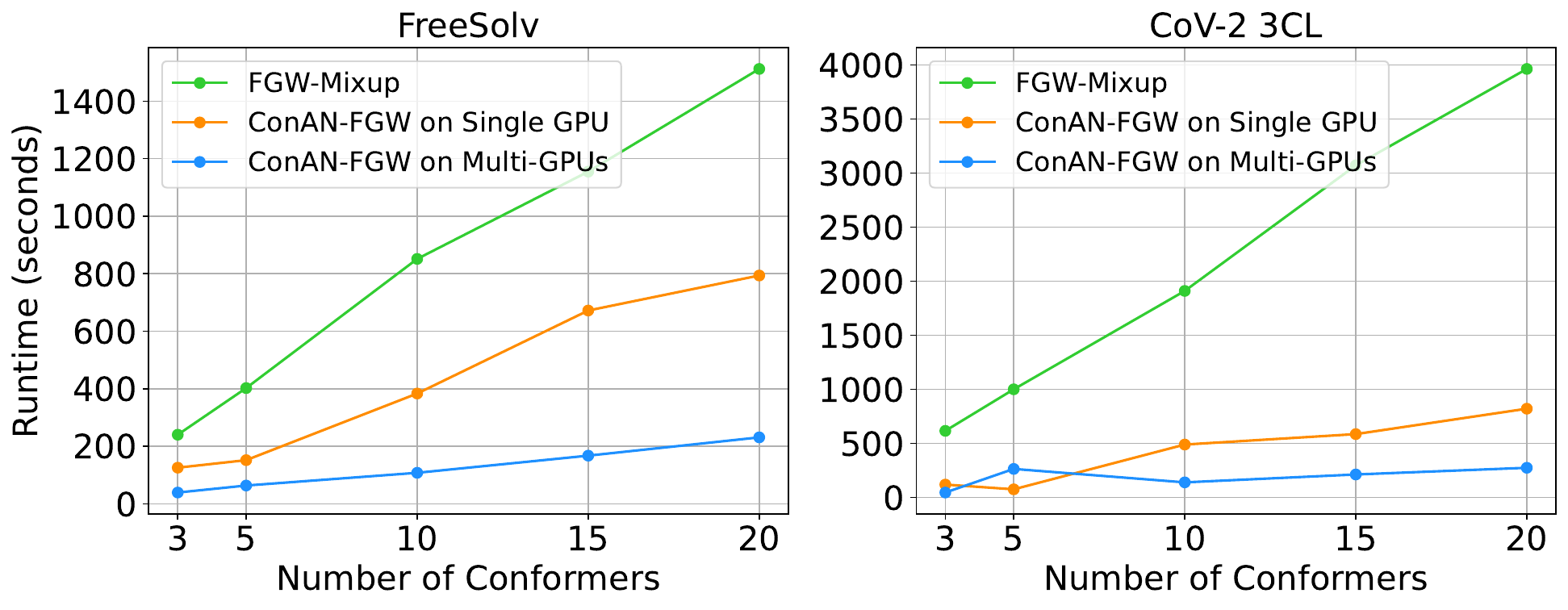}
\vspace{-0.15in}
      \caption{Comparing runtimes of FGW-Mixup, \conan-FGW (single and multi-GPU).}
      \label{fig:runtime_freesolv_cov23cl}
\end{figure}
\vspace{-0.3in}
\section{Conclusion and Future Works}
\vspace{-0.05in}
In this study, we present an $E(3)$-invariant molecular conformer aggregation network that
integrates 2D molecular graphs, 3D conformers, and geometry-attributed structures using Fused Gromov-Wasserstein barycenter formulations. The results indicate the effectiveness of this approach, surpassing several baseline methods across diverse downstream tasks, including molecular property prediction and 3D classification. Moreover, we investigate the convergence properties of the empirical barycenter problem, demonstrating that an adequate number of conformers can yield a reliable approximation of the target structure. To enable training on large datasets, we also introduce entropic barycenter solvers, maximizing GPU utilization. Future research directions include exploring the robustness of using RDKit for multiple low-energy scenarios or more accurate reference methods for atomic structure relaxation, such as density-functional theory. Finally, extending \conan, to learn from large-scale unlabeled multi-modal molecular datasets holds significant promise for advancing the field.

\section*{Acknowledgement}
\vspace{-0.05in}
The authors thank the
International Max Planck Research School for Intelligent Systems (IMPRS-IS) for supporting Duy
M. H. Nguyen and Nina Lukashina. Duy
M. H. Nguyen and Daniel Sonntag are also supported by the XAINES project (BMBF, 01IW20005), No-IDLE project (BMBF, 01IW23002), and the Endowed Chair of Applied Artificial Intelligence, Oldenburg
University. An T. Le was supported by the German Research Foundation project METRIC4IMITATION (PE 2315/11-1). Nhat Ho acknowledges support from the NSF IFML 2019844 and the NSF AI Institute for Foundations of Machine Learning. Mathias Niepert acknowledges funding by Deutsche Forschungsgemeinschaft (DFG, German Research Foundation) under Germany’s Excellence Strategy - EXC and support by the Stuttgart Center for Simulation Science (SimTech).  Furthermore, we acknowledge the support of the European Laboratory for Learning and Intelligent Systems (ELLIS) Unit Stuttgart. 

\section*{Impact Statement}
\vspace{-0.05in}
This paper presents work whose goal is to advance the field of Machine Learning, focusing on molecular representation learning. There are many potential societal consequences of our work, none of which we feel must be specifically highlighted here.

\bibliography{ICML2024/main}

\begin{thebibliography}{85}
\providecommand{\natexlab}[1]{#1}
\providecommand{\url}[1]{\texttt{#1}}
\expandafter\ifx\csname urlstyle\endcsname\relax
  \providecommand{\doi}[1]{doi: #1}\else
  \providecommand{\doi}{doi: \begingroup \urlstyle{rm}\Url}\fi

\bibitem[Ahmad et~al.(2022)Ahmad, Simon, Chithrananda, Grand, and
  Ramsundar]{ChemBerta2}
Ahmad, W., Simon, E., Chithrananda, S., Grand, G., and Ramsundar, B.
\newblock Chemberta-2: Towards chemical foundation models, 2022.

\bibitem[Axelrod \& Gómez-Bombarelli(2023)Axelrod and
  Gómez-Bombarelli]{axelrod_molecular_2023}
Axelrod, S. and Gómez-Bombarelli, R.
\newblock Molecular machine learning with conformer ensembles.
\newblock \emph{Mach. Learn.: Sci. Technol.}, 4\penalty0 (3):\penalty0 035025,
  September 2023.
\newblock ISSN 2632-2153.
\newblock \doi{10.1088/2632-2153/acefa7}.

\bibitem[Batatia et~al.(2022)Batatia, Kovacs, Simm, Ortner, and
  Csanyi]{batatia2022mace}
Batatia, I., Kovacs, D.~P., Simm, G., Ortner, C., and Csanyi, G.
\newblock Mace: Higher order equivariant message passing neural networks for
  fast and accurate force fields.
\newblock In Koyejo, S., Mohamed, S., Agarwal, A., Belgrave, D., Cho, K., and
  Oh, A. (eds.), \emph{Advances in Neural Information Processing Systems},
  volume~35, pp.\  11423--11436. Curran Associates, Inc., 2022.

\bibitem[Batatia et~al.(2023)Batatia, Benner, Chiang, Elena, Kovács,
  Riebesell, Advincula, Asta, Baldwin, Bernstein, Bhowmik, Blau, Cărare,
  Darby, De, Pia, Deringer, Elijošius, El-Machachi, Fako, Ferrari,
  Genreith-Schriever, George, Goodall, Grey, Han, Handley, Heenen, Hermansson,
  Holm, Jaafar, Hofmann, Jakob, Jung, Kapil, Kaplan, Karimitari, Kroupa,
  Kullgren, Kuner, Kuryla, Liepuoniute, Margraf, Magdău, Michaelides, Moore,
  Naik, Niblett, Norwood, O'Neill, Ortner, Persson, Reuter, Rosen, Schaaf,
  Schran, Sivonxay, Stenczel, Svahn, Sutton, van~der Oord, Varga-Umbrich,
  Vegge, Vondrák, Wang, Witt, Zills, and Csányi]{Batatia2023}
Batatia, I., Benner, P., Chiang, Y., Elena, A.~M., Kovács, D.~P., Riebesell,
  J., Advincula, X.~R., Asta, M., Baldwin, W.~J., Bernstein, N., Bhowmik, A.,
  Blau, S.~M., Cărare, V., Darby, J.~P., De, S., Pia, F.~D., Deringer, V.~L.,
  Elijošius, R., El-Machachi, Z., Fako, E., Ferrari, A.~C.,
  Genreith-Schriever, A., George, J., Goodall, R. E.~A., Grey, C.~P., Han, S.,
  Handley, W., Heenen, H.~H., Hermansson, K., Holm, C., Jaafar, J., Hofmann,
  S., Jakob, K.~S., Jung, H., Kapil, V., Kaplan, A.~D., Karimitari, N., Kroupa,
  N., Kullgren, J., Kuner, M.~C., Kuryla, D., Liepuoniute, G., Margraf, J.~T.,
  Magdău, I.-B., Michaelides, A., Moore, J.~H., Naik, A.~A., Niblett, S.~P.,
  Norwood, S.~W., O'Neill, N., Ortner, C., Persson, K.~A., Reuter, K., Rosen,
  A.~S., Schaaf, L.~L., Schran, C., Sivonxay, E., Stenczel, T.~K., Svahn, V.,
  Sutton, C., van~der Oord, C., Varga-Umbrich, E., Vegge, T., Vondrák, M.,
  Wang, Y., Witt, W.~C., Zills, F., and Csányi, G.
\newblock A foundation model for atomistic materials chemistry, 2023.

\bibitem[Batzner et~al.(2022)Batzner, Musaelian, Sun, Geiger, Mailoa,
  Kornbluth, Molinari, Smidt, and Kozinsky]{Batzner2022}
Batzner, S., Musaelian, A., Sun, L., Geiger, M., Mailoa, J.~P., Kornbluth, M.,
  Molinari, N., Smidt, T.~E., and Kozinsky, B.
\newblock E(3)-equivariant graph neural networks for data-efficient and
  accurate interatomic potentials.
\newblock \emph{Nat. Commun.}, 13\penalty0 (1):\penalty0 2453, May 2022.
\newblock ISSN 2041-1723.

\bibitem[Brandstetter et~al.(2022)Brandstetter, Hesselink, van~der Pol,
  Bekkers, and Welling]{brandstetter2022geometric}
Brandstetter, J., Hesselink, R., van~der Pol, E., Bekkers, E.~J., and Welling,
  M.
\newblock Geometric and physical quantities improve e(3) equivariant message
  passing.
\newblock In \emph{International Conference on Learning Representations}, 2022.

\bibitem[Brogat-Motte et~al.(2022)Brogat-Motte, Flamary, Brouard, Rousu, and
  D'Alché-Buc]{brogat_learning_2022}
Brogat-Motte, L., Flamary, R., Brouard, C., Rousu, J., and D'Alché-Buc, F.
\newblock Learning to {Predict} {Graphs} with {Fused} {Gromov}-{Wasserstein}
  {Barycenters}.
\newblock In Chaudhuri, K., Jegelka, S., Song, L., Szepesvari, C., Niu, G., and
  Sabato, S. (eds.), \emph{Proceedings of the 39th {International} {Conference}
  on {Machine} {Learning}}, volume 162 of \emph{Proceedings of {Machine}
  {Learning} {Research}}, pp.\  2321--2335. PMLR, July 2022.

\bibitem[Butler et~al.(2018)Butler, Davies, Cartwright, Isayev, and
  Walsh]{Butler2018}
Butler, K.~T., Davies, D.~W., Cartwright, H., Isayev, O., and Walsh, A.
\newblock Machine learning for molecular and materials science.
\newblock \emph{Nature}, 559\penalty0 (7715):\penalty0 547--555, Jul 2018.
\newblock ISSN 1476-4687.

\bibitem[Cao et~al.(2022)Cao, Coventry, Goreshnik, Huang, Sheffler, Park, Jude,
  Markovi{\'c}, Kadam, Verschueren, et~al.]{cao2022design}
Cao, L., Coventry, B., Goreshnik, I., Huang, B., Sheffler, W., Park, J.~S.,
  Jude, K.~M., Markovi{\'c}, I., Kadam, R.~U., Verschueren, K.~H., et~al.
\newblock Design of protein-binding proteins from the target structure alone.
\newblock \emph{Nature}, 605\penalty0 (7910):\penalty0 551--560, 2022.

\bibitem[Chang \& Ye(2023)Chang and Ye]{chang2023bidirectional}
Chang, J. and Ye, J.~C.
\newblock Bidirectional generation of structure and properties through a single
  molecular foundation model.
\newblock \emph{arXiv preprint arXiv:2211.10590}, 2023.

\bibitem[Chen et~al.(2020)Chen, B{\'e}cigneul, Ganea, Barzilay, and
  Jaakkola]{chen2020optimal}
Chen, B., B{\'e}cigneul, G., Ganea, O.-E., Barzilay, R., and Jaakkola, T.
\newblock Optimal transport graph neural networks.
\newblock \emph{arXiv preprint arXiv:2006.04804}, 2020.

\bibitem[Choudhary et~al.(2022)Choudhary, DeCost, Chen, Jain, Tavazza, Cohn,
  Park, Choudhary, Agrawal, Billinge, Holm, Ong, and Wolverton]{Choudhary2022}
Choudhary, K., DeCost, B., Chen, C., Jain, A., Tavazza, F., Cohn, R., Park,
  C.~W., Choudhary, A., Agrawal, A., Billinge, S. J.~L., Holm, E., Ong, S.~P.,
  and Wolverton, C.
\newblock Recent advances and applications of deep learning methods in
  materials science.
\newblock \emph{npj Comput. Mater.}, 8\penalty0 (1):\penalty0 59, Apr 2022.
\newblock ISSN 2057-3960.

\bibitem[Chowdhury \& Needham(2021)Chowdhury and
  Needham]{chowdhury2021generalized}
Chowdhury, S. and Needham, T.
\newblock Generalized spectral clustering via gromov-wasserstein learning.
\newblock In \emph{International Conference on Artificial Intelligence and
  Statistics}, pp.\  712--720. PMLR, 2021.

\bibitem[Cuturi(2013)]{cuturi2013sinkhorn}
Cuturi, M.
\newblock Sinkhorn distances: Lightspeed computation of optimal transport.
\newblock \emph{Advances in neural information processing systems}, 26, 2013.

\bibitem[Cuturi \& Doucet(2014)Cuturi and Doucet]{cuturi2014fast}
Cuturi, M. and Doucet, A.
\newblock Fast computation of wasserstein barycenters.
\newblock In \emph{International conference on machine learning}, pp.\
  685--693. PMLR, 2014.

\bibitem[Du et~al.(2022)Du, Zhang, Du, Meng, Chen, Zheng, Shao, and
  Liu]{du2022se}
Du, W., Zhang, H., Du, Y., Meng, Q., Chen, W., Zheng, N., Shao, B., and Liu,
  T.-Y.
\newblock Se (3) equivariant graph neural networks with complete local frames.
\newblock In \emph{International Conference on Machine Learning}, pp.\
  5583--5608. PMLR, 2022.

\bibitem[Du et~al.(2024)Du, Wang, Feng, Wang, Ji, Gomes, Ma, et~al.]{du2024new}
Du, Y., Wang, L., Feng, D., Wang, G., Ji, S., Gomes, C.~P., Ma, Z.-M., et~al.
\newblock A new perspective on building efficient and expressive 3d equivariant
  graph neural networks.
\newblock \emph{Advances in Neural Information Processing Systems}, 36, 2024.

\bibitem[Ellinger et~al.(2020)Ellinger, Bojkova, Zaliani, Cinatl, Claussen,
  Westhaus, Reinshagen, Kuzikov, Wolf, Geisslinger, Gribbon, and
  Ciesek]{PPR:PPR153247}
Ellinger, B., Bojkova, D., Zaliani, A., Cinatl, J., Claussen, C., Westhaus, S.,
  Reinshagen, J., Kuzikov, M., Wolf, M., Geisslinger, G., Gribbon, P., and
  Ciesek, S.
\newblock Identification of inhibitors of sars-cov-2 in-vitro cellular toxicity
  in human (caco-2) cells using a large scale drug repurposing collection,
  2020.

\bibitem[Fang et~al.(2021)Fang, Liu, Lei, He, Zhang, Zhou, Wang, Wu, and
  Wang]{ChemRL-GEM}
Fang, X., Liu, L., Lei, J., He, D., Zhang, S., Zhou, J., Wang, F., Wu, H., and
  Wang, H.
\newblock Chemrl-gem: Geometry enhanced molecular representation learning for
  property prediction.
\newblock \emph{Nature Machine Intelligence}, 2021.
\newblock \doi{10.48550/ARXIV.2106.06130}.

\bibitem[Fang et~al.(2022)Fang, Liu, Lei, He, Zhang, Zhou, Wang, Wu, and
  Wang]{fang2022geometry}
Fang, X., Liu, L., Lei, J., He, D., Zhang, S., Zhou, J., Wang, F., Wu, H., and
  Wang, H.
\newblock Geometry-enhanced molecular representation learning for property
  prediction.
\newblock \emph{Nature Machine Intelligence}, 4\penalty0 (2):\penalty0
  127--134, 2022.

\bibitem[Fedik et~al.(2022)Fedik, Zubatyuk, Kulichenko, Lubbers, Smith, Nebgen,
  Messerly, Li, Boldyrev, Barros, Isayev, and Tretiak]{Fedik2022}
Fedik, N., Zubatyuk, R., Kulichenko, M., Lubbers, N., Smith, J.~S., Nebgen, B.,
  Messerly, R., Li, Y.~W., Boldyrev, A.~I., Barros, K., Isayev, O., and
  Tretiak, S.
\newblock Extending machine learning beyond interatomic potentials for
  predicting molecular properties.
\newblock \emph{Nat. Rev. Chem.}, 6\penalty0 (9):\penalty0 653--672, Sep 2022.
\newblock ISSN 2397-3358.
\newblock \doi{10.1038/s41570-022-00416-3}.

\bibitem[Feydy et~al.(2019)Feydy, S{\'e}journ{\'e}, Vialard, Amari, Trouv{\'e},
  and Peyr{\'e}]{feydy2019interpolating}
Feydy, J., S{\'e}journ{\'e}, T., Vialard, F.-X., Amari, S.-i., Trouv{\'e}, A.,
  and Peyr{\'e}, G.
\newblock Interpolating between optimal transport and mmd using sinkhorn
  divergences.
\newblock In \emph{The 22nd International Conference on aRtIfIcIaL InTeLlIgEnCe
  and Statistics}, pp.\  2681--2690. PMLR, 2019.

\bibitem[Fuchs et~al.(2020)Fuchs, Worrall, Fischer, and Welling]{fuchs2020se}
Fuchs, F., Worrall, D., Fischer, V., and Welling, M.
\newblock Se(3)-transformers: 3d roto-translation equivariant attention
  networks.
\newblock In Larochelle, H., Ranzato, M., Hadsell, R., Balcan, M., and Lin, H.
  (eds.), \emph{Advances in Neural Information Processing Systems}, volume~33,
  pp.\  1970--1981. Curran Associates, Inc., 2020.

\bibitem[Gasteiger et~al.(2020{\natexlab{a}})Gasteiger, Gro{\ss}, and
  G{\"u}nnemann]{gasteiger2020directional}
Gasteiger, J., Gro{\ss}, J., and G{\"u}nnemann, S.
\newblock Directional message passing for molecular graphs.
\newblock \emph{International Conference on Learning Representations},
  2020{\natexlab{a}}.

\bibitem[Gasteiger et~al.(2020{\natexlab{b}})Gasteiger, Gro{\ss}, and
  G{\"u}nnemann]{gasteiger_dimenet_2020}
Gasteiger, J., Gro{\ss}, J., and G{\"u}nnemann, S.
\newblock Directional message passing for molecular graphs.
\newblock In \emph{International Conference on Learning Representations
  (ICLR)}, 2020{\natexlab{b}}.

\bibitem[Gasteiger et~al.(2021)Gasteiger, Becker, and
  G{\"u}nnemann]{gasteiger2021gemnet}
Gasteiger, J., Becker, F., and G{\"u}nnemann, S.
\newblock Gemnet: Universal directional graph neural networks for molecules.
\newblock In Beygelzimer, A., Dauphin, Y., Liang, P., and Vaughan, J.~W.
  (eds.), \emph{Advances in Neural Information Processing Systems}, 2021.

\bibitem[Gilmer et~al.(2017{\natexlab{a}})Gilmer, Schoenholz, Riley, Vinyals,
  and Dahl]{Gil+2017}
Gilmer, J., Schoenholz, S.~S., Riley, P.~F., Vinyals, O., and Dahl, G.~E.
\newblock Neural message passing for quantum chemistry.
\newblock In \emph{International Conference on Machine Learning}, pp.\
  1263--1272, 2017{\natexlab{a}}.

\bibitem[Gilmer et~al.(2017{\natexlab{b}})Gilmer, Schoenholz, Riley, Vinyals,
  and Dahl]{Gilmer2017}
Gilmer, J., Schoenholz, S.~S., Riley, P.~F., Vinyals, O., and Dahl, G.~E.
\newblock Neural message passing for quantum chemistry.
\newblock In Precup, D. and Teh, Y.~W. (eds.), \emph{Proceedings of the 34th
  ICML}, volume~70 of \emph{Proceedings of Machine Learning Research}, pp.\
  1263--1272. PMLR, 06--11 Aug 2017{\natexlab{b}}.

\bibitem[Grimme(2019)]{grimme2019exploration}
Grimme, S.
\newblock Exploration of chemical compound, conformer, and reaction space with
  meta-dynamics simulations based on tight-binding quantum chemical
  calculations.
\newblock \emph{Journal of chemical theory and computation}, 15\penalty0
  (5):\penalty0 2847--2862, 2019.

\bibitem[Hawkins(2017)]{hawkins2017conformation}
Hawkins, P.~C.
\newblock Conformation generation: the state of the art.
\newblock \emph{Journal of chemical information and modeling}, 57\penalty0
  (8):\penalty0 1747--1756, 2017.

\bibitem[Hu et~al.(2020{\natexlab{a}})Hu, Liu, Gomes, Zitnik, Liang, Pande, and
  Leskovec]{Hu2020Strategies}
Hu, W., Liu, B., Gomes, J., Zitnik, M., Liang, P., Pande, V., and Leskovec, J.
\newblock Strategies for pre-training graph neural networks.
\newblock In \emph{International Conference on Learning Representations},
  2020{\natexlab{a}}.

\bibitem[Hu et~al.(2020{\natexlab{b}})Hu, Liu, Gomes, Zitnik, Liang, Pande, and
  Leskovec]{hu2020pretraining}
Hu, W., Liu, B., Gomes, J., Zitnik, M., Liang, P., Pande, V., and Leskovec, J.
\newblock Strategies for pre-training graph neural networks.
\newblock In \emph{International Conference on Learning Representations},
  2020{\natexlab{b}}.

\bibitem[Hu et~al.(2021)Hu, Fey, Ren, Nakata, Dong, and Leskovec]{hu2021ogblsc}
Hu, W., Fey, M., Ren, H., Nakata, M., Dong, Y., and Leskovec, J.
\newblock Ogb-lsc: A large-scale challenge for machine learning on graphs.
\newblock \emph{arXiv preprint arXiv:2103.09430}, 2021.

\bibitem[Kipf \& Welling(2017)Kipf and Welling]{Kipf2017}
Kipf, T.~N. and Welling, M.
\newblock Semi-supervised classification with graph convolutional networks.
\newblock In \emph{International Conference on Learning Representations}, 2017.

\bibitem[Landrum(2016)]{landrum2016rdkit}
Landrum, G.
\newblock Rdkit: open-source cheminformatics http://www. rdkit. org.
\newblock 3\penalty0 (8), 2016.

\bibitem[Le et~al.(2022)Le, Le, Nguyen, Do, Pham, and Ho]{Ho_Gromov_closeform}
Le, K., Le, D., Nguyen, H., Do, D., Pham, T., and Ho, N.
\newblock Entropic {G}romov-{W}asserstein between {G}aussian distributions.
\newblock In \emph{ICML}, 2022.

\bibitem[Le~Gouic et~al.(2022)Le~Gouic, Paris, Rigollet, and
  Stromme]{le_gouic_fast_2022}
Le~Gouic, T., Paris, Q., Rigollet, P., and Stromme, A.~J.
\newblock Fast convergence of empirical barycenters in {Alexandrov} spaces and
  the {Wasserstein} space.
\newblock \emph{Journal of the European Mathematical Society}, 25\penalty0
  (6):\penalty0 2229--2250, May 2022.
\newblock ISSN 1435-9855.

\bibitem[Lin et~al.(2020)Lin, Ho, Chen, Cuturi, and
  Jordan]{Lin-2020-Revisiting}
Lin, T., Ho, N., Chen, X., Cuturi, M., and Jordan, M.~I.
\newblock Fixed-support {W}asserstein barycenters: Computational hardness and
  fast algorithm.
\newblock In \emph{NeurIPS}, pp.\  5368--5380, 2020.

\bibitem[Liu et~al.(2021)Liu, Fu, Zhang, Wang, Xie, Yuan, Luo, Xu, Xu, and
  Ji]{liu2021fast}
Liu, M., Fu, C., Zhang, X., Wang, L., Xie, Y., Yuan, H., Luo, Y., Xu, Z., Xu,
  S., and Ji, S.
\newblock Fast quantum property prediction via deeper 2d and 3d graph networks.
\newblock \emph{arXiv preprint arXiv:2106.08551}, 2021.

\bibitem[Liu et~al.(2022{\natexlab{a}})Liu, Wang, Liu, Lasenby, Guo, and
  Tang]{liu2022pretraining}
Liu, S., Wang, H., Liu, W., Lasenby, J., Guo, H., and Tang, J.
\newblock Pre-training molecular graph representation with 3d geometry.
\newblock In \emph{International Conference on Learning Representations},
  2022{\natexlab{a}}.

\bibitem[Liu et~al.(2022{\natexlab{b}})Liu, Wang, Liu, Lin, Zhang, Oztekin, and
  Ji]{SphereNet}
Liu, Y., Wang, L., Liu, M., Lin, Y., Zhang, X., Oztekin, B., and Ji, S.
\newblock Spherical message passing for 3d molecular graphs.
\newblock In \emph{International Conference on Learning Representations},
  2022{\natexlab{b}}.

\bibitem[Ma et~al.(2023)Ma, Chu, Wang, Lin, Zhao, Ma, and Zhu]{ma2023fused}
Ma, X., Chu, X., Wang, Y., Lin, Y., Zhao, J., Ma, L., and Zhu, W.
\newblock Fused gromov-wasserstein graph mixup for graph-level classifications.
\newblock In \emph{Thirty-seventh Conference on Neural Information Processing
  Systems}, 2023.

\bibitem[Meyer et~al.(2018)Meyer, Sawatlon, Heinen, Von~Lilienfeld, and
  Corminboeuf]{meyer2018machine}
Meyer, B., Sawatlon, B., Heinen, S., Von~Lilienfeld, O.~A., and Corminboeuf, C.
\newblock Machine learning meets volcano plots: computational discovery of
  cross-coupling catalysts.
\newblock \emph{Chemical science}, 9\penalty0 (35):\penalty0 7069--7077, 2018.

\bibitem[Morgan(1965)]{Morgan1965}
Morgan, H.~L.
\newblock The generation of a unique machine description for chemical
  structures-a technique developed at chemical abstracts service.
\newblock \emph{Journal of Chemical Documentation}, 5\penalty0 (2):\penalty0
  107–113, May 1965.
\newblock ISSN 1541-5732.
\newblock \doi{10.1021/c160017a018}.

\bibitem[Neyshabur et~al.(2013)Neyshabur, Khadem, Hashemifar, and
  Arab]{neyshabur2013netal}
Neyshabur, B., Khadem, A., Hashemifar, S., and Arab, S.~S.
\newblock Netal: a new graph-based method for global alignment of
  protein--protein interaction networks.
\newblock \emph{Bioinformatics}, 29\penalty0 (13):\penalty0 1654--1662, 2013.

\bibitem[Paszke et~al.(2017)Paszke, Gross, Chintala, Chanan, Yang, DeVito, Lin,
  Desmaison, Antiga, and Lerer]{paszke2017automatic}
Paszke, A., Gross, S., Chintala, S., Chanan, G., Yang, E., DeVito, Z., Lin, Z.,
  Desmaison, A., Antiga, L., and Lerer, A.
\newblock Automatic differentiation in pytorch.
\newblock In \emph{NIPS 2017 Workshop on Autodiff}, 2017.

\bibitem[Peyr{\'e}(2015)]{peyre2015entropic}
Peyr{\'e}, G.
\newblock Entropic approximation of wasserstein gradient flows.
\newblock \emph{SIAM Journal on Imaging Sciences}, 8\penalty0 (4):\penalty0
  2323--2351, 2015.

\bibitem[Peyr{\'e} et~al.(2016)Peyr{\'e}, Cuturi, and Solomon]{peyre2016gromov}
Peyr{\'e}, G., Cuturi, M., and Solomon, J.
\newblock Gromov-wasserstein averaging of kernel and distance matrices.
\newblock In \emph{International conference on machine learning}, pp.\
  2664--2672. PMLR, 2016.

\bibitem[Peyr{\'e} et~al.(2019)Peyr{\'e}, Cuturi,
  et~al.]{peyre2019computational}
Peyr{\'e}, G., Cuturi, M., et~al.
\newblock Computational optimal transport: With applications to data science.
\newblock \emph{Foundations and Trends{\textregistered} in Machine Learning},
  11\penalty0 (5-6):\penalty0 355--607, 2019.

\bibitem[Peyré et~al.(2016)Peyré, Cuturi, and
  Solomon]{peyre_gromov_wasserstein_2016}
Peyré, G., Cuturi, M., and Solomon, J.
\newblock Gromov-{Wasserstein} {Averaging} of {Kernel} and {Distance}
  {Matrices}.
\newblock In Balcan, M.~F. and Weinberger, K.~Q. (eds.), \emph{Proceedings of
  {The} 33rd {International} {Conference} on {Machine} {Learning}}, volume~48
  of \emph{Proceedings of {Machine} {Learning} {Research}}, pp.\  2664--2672,
  New York, New York, USA, June 2016. PMLR.

\bibitem[Rioux et~al.(2023)Rioux, Goldfeld, and Kato]{rioux2023entropic}
Rioux, G., Goldfeld, Z., and Kato, K.
\newblock Entropic gromov-wasserstein distances: Stability, algorithms, and
  distributional limits.
\newblock \emph{arXiv preprint arXiv:2306.00182}, 2023.

\bibitem[Rong et~al.(2020)Rong, Bian, Xu, Xie, Wei, Huang, and
  Huang]{rong2020self}
Rong, Y., Bian, Y., Xu, T., Xie, W., Wei, Y., Huang, W., and Huang, J.
\newblock Self-supervised graph transformer on large-scale molecular data.
\newblock \emph{Advances in Neural Information Processing Systems},
  33:\penalty0 12559--12571, 2020.

\bibitem[Ross et~al.(2022)Ross, Belgodere, Chenthamarakshan, Padhi, Mroueh, and
  Das]{ross2022large}
Ross, J., Belgodere, B., Chenthamarakshan, V., Padhi, I., Mroueh, Y., and Das,
  P.
\newblock Large-scale chemical language representations capture molecular
  structure and properties.
\newblock \emph{Nature Machine Intelligence}, 4\penalty0 (12):\penalty0
  1256--1264, 2022.

\bibitem[Satorras et~al.(2021)Satorras, Hoogeboom, and Welling]{satorras2021En}
Satorras, V.~G., Hoogeboom, E., and Welling, M.
\newblock E(n) equivariant graph neural networks.
\newblock In Meila, M. and Zhang, T. (eds.), \emph{Proceedings of the 38th
  International Conference on Machine Learning}, volume 139 of
  \emph{Proceedings of Machine Learning Research}, pp.\  9323--9332. PMLR,
  18--24 Jul 2021.

\bibitem[Scarselli et~al.(2009)Scarselli, Gori, Tsoi, Hagenbuchner, and
  Monfardini]{Sca+2009}
Scarselli, F., Gori, M., Tsoi, A.~C., Hagenbuchner, M., and Monfardini, G.
\newblock The graph neural network model.
\newblock \emph{IEEE Transactions on Neural Networks}, 20\penalty0
  (1):\penalty0 61--80, 2009.

\bibitem[Schmitzer(2019)]{schmitzer2019stabilized}
Schmitzer, B.
\newblock Stabilized sparse scaling algorithms for entropy regularized
  transport problems.
\newblock \emph{SIAM Journal on Scientific Computing}, 41\penalty0
  (3):\penalty0 A1443--A1481, 2019.

\bibitem[Sch{\"{u}}tt et~al.(2017)Sch{\"{u}}tt, Kindermans, Felix, Chmiela,
  Tkatchenko, and M{\"{u}}ller]{Sch+2017}
Sch{\"{u}}tt, K., Kindermans, P., Felix, H. E.~S., Chmiela, S., Tkatchenko, A.,
  and M{\"{u}}ller, K.
\newblock Schnet: {A} continuous-filter convolutional neural network for
  modeling quantum interactions.
\newblock In \emph{Advances in Neural Information Processing Systems 30: Annual
  Conference on Neural Information Processing Systems 2017, December 4-9, 2017,
  Long Beach, CA, {USA}}, pp.\  991--1001, 2017.

\bibitem[Sch\"{u}tt et~al.(2017)Sch\"{u}tt, Kindermans, Sauceda~Felix, Chmiela,
  Tkatchenko, and M\"{u}ller]{schutt2017schnet}
Sch\"{u}tt, K., Kindermans, P.-J., Sauceda~Felix, H.~E., Chmiela, S.,
  Tkatchenko, A., and M\"{u}ller, K.-R.
\newblock Schnet: A continuous-filter convolutional neural network for modeling
  quantum interactions.
\newblock In Guyon, I., Luxburg, U.~V., Bengio, S., Wallach, H., Fergus, R.,
  Vishwanathan, S., and Garnett, R. (eds.), \emph{Advances in Neural
  Information Processing Systems}, volume~30. Curran Associates, Inc., 2017.

\bibitem[Sch{\"u}tt et~al.(2021)Sch{\"u}tt, Unke, and Gastegger]{Schuett2021}
Sch{\"u}tt, K.~T., Unke, O.~T., and Gastegger, M.
\newblock Equivariant message passing for the prediction of tensorial
  properties and molecular spectra.
\newblock \emph{ICML}, pp.\  1--13, 2021.

\bibitem[Source(2020)]{diamondXChem}
Source, D.~L.
\newblock Main protease structure and xchem fragment screen, 2020.

\bibitem[St{\"a}rk et~al.(2022)St{\"a}rk, Beaini, Corso, Tossou, Dallago,
  G{\"u}nnemann, and Li{\'o}]{pmlr-v162-stark22a}
St{\"a}rk, H., Beaini, D., Corso, G., Tossou, P., Dallago, C., G{\"u}nnemann,
  S., and Li{\'o}, P.
\newblock 3{D} infomax improves {GNN}s for molecular property prediction.
\newblock In Chaudhuri, K., Jegelka, S., Song, L., Szepesvari, C., Niu, G., and
  Sabato, S. (eds.), \emph{Proceedings of the 39th International Conference on
  Machine Learning}, volume 162 of \emph{Proceedings of Machine Learning
  Research}, pp.\  20479--20502. PMLR, 17--23 Jul 2022.

\bibitem[Tang et~al.(2023)Tang, Zhao, and Li]{tang2023fused}
Tang, J., Zhao, K., and Li, J.
\newblock A fused gromov-wasserstein framework for unsupervised knowledge graph
  entity alignment.
\newblock \emph{arXiv preprint arXiv:2305.06574}, 2023.

\bibitem[Titouan et~al.(2019)Titouan, Courty, Tavenard, Laetitia, and
  Flamary]{titouan_optimal_2019}
Titouan, V., Courty, N., Tavenard, R., Laetitia, C., and Flamary, R.
\newblock Optimal {Transport} for structured data with application on graphs.
\newblock In Chaudhuri, K. and Salakhutdinov, R. (eds.), \emph{Proceedings of
  the 36th {International} {Conference} on {Machine} {Learning}}, volume~97 of
  \emph{Proceedings of {Machine} {Learning} {Research}}, pp.\  6275--6284.
  PMLR, June 2019.

\bibitem[Titouan et~al.(2020)Titouan, Chapel, Flamary, Tavenard, and
  Courty]{titouan_fused_2020}
Titouan, V., Chapel, L., Flamary, R., Tavenard, R., and Courty, N.
\newblock Fused {Gromov}-{Wasserstein} {Distance} for {Structured} {Objects}.
\newblock \emph{Algorithms}, 13\penalty0 (9):\penalty0 212, August 2020.
\newblock ISSN 1999-4893.
\newblock \doi{10.3390/a13090212}.

\bibitem[Touret et~al.(2020)Touret, Gilles, Barral, and et~al.]{Touret2020}
Touret, F., Gilles, M., Barral, K., and et~al.
\newblock In vitro screening of a fda approved chemical library reveals
  potential inhibitors of sars-cov-2 replication.
\newblock \emph{Sci Rep}, 10:\penalty0 13093, 2020.
\newblock \doi{10.1038/s41598-020-70143-6}.

\bibitem[Vamathevan et~al.(2019)Vamathevan, Clark, Czodrowski, Dunham, Ferran,
  Lee, Li, Madabhushi, Shah, Spitzer, and Zhao]{Vamathevan2019}
Vamathevan, J., Clark, D., Czodrowski, P., Dunham, I., Ferran, E., Lee, G., Li,
  B., Madabhushi, A., Shah, P., Spitzer, M., and Zhao, S.
\newblock Applications of machine learning in drug discovery and development.
\newblock \emph{Nat. Rev. Drug Discov.}, 18\penalty0 (6):\penalty0 463--477,
  Jun 2019.
\newblock ISSN 1474-1784.

\bibitem[Veli\v{c}kovi\'{c} et~al.(2018)Veli\v{c}kovi\'{c}, Cucurull, Casanova,
  Romero, Li{\`{o}}, and Bengio]{Vel+2018}
Veli\v{c}kovi\'{c}, P., Cucurull, G., Casanova, A., Romero, A., Li{\`{o}}, P.,
  and Bengio, Y.
\newblock Graph attention networks.
\newblock In \emph{International Conference on Learning Representations}, 2018.

\bibitem[Veličković et~al.(2018)Veličković, Cucurull, Casanova, Romero,
  Liò, and Bengio]{Velickovic2018}
Veličković, P., Cucurull, G., Casanova, A., Romero, A., Liò, P., and Bengio,
  Y.
\newblock Graph attention networks.
\newblock In \emph{ICLR}, 2018.

\bibitem[Vincent-Cuaz et~al.(2021)Vincent-Cuaz, Vayer, Flamary, Corneli, and
  Courty]{vincent2021online}
Vincent-Cuaz, C., Vayer, T., Flamary, R., Corneli, M., and Courty, N.
\newblock Online graph dictionary learning.
\newblock In \emph{International conference on machine learning}, pp.\
  10564--10574. PMLR, 2021.

\bibitem[Vincent-Cuaz et~al.(2022)Vincent-Cuaz, Flamary, Corneli, Vayer, and
  Courty]{vincent_cuaz_template_2022}
Vincent-Cuaz, C., Flamary, R., Corneli, M., Vayer, T., and Courty, N.
\newblock Template based {Graph} {Neural} {Network} with {Optimal} {Transport}
  {Distances}.
\newblock In Koyejo, S., Mohamed, S., Agarwal, A., Belgrave, D., Cho, K., and
  Oh, A. (eds.), \emph{Advances in {Neural} {Information} {Processing}
  {Systems}}, volume~35, pp.\  11800--11814. Curran Associates, Inc., 2022.

\bibitem[Wang et~al.(2019)Wang, Guo, Wang, Sun, and Huang]{SmilesBert}
Wang, S., Guo, Y., Wang, Y., Sun, H., and Huang, J.
\newblock Smiles-bert: Large scale unsupervised pre-training for molecular
  property prediction.
\newblock In \emph{Proceedings of the 10th ACM International Conference on
  Bioinformatics, Computational Biology and Health Informatics}, BCB '19, pp.\
  429–436, New York, NY, USA, 2019. Association for Computing Machinery.
\newblock ISBN 9781450366663.
\newblock \doi{10.1145/3307339.3342186}.

\bibitem[Wang et~al.(2022)Wang, Wang, Cao, and
  Barati~Farimani]{wang2022molecular}
Wang, Y., Wang, J., Cao, Z., and Barati~Farimani, A.
\newblock Molecular contrastive learning of representations via graph neural
  networks.
\newblock \emph{Nature Machine Intelligence}, 4\penalty0 (3):\penalty0
  279--287, 2022.

\bibitem[Wang et~al.(2024{\natexlab{a}})Wang, Wang, Li, He, Li, Wang, Zheng,
  Shao, and Liu]{wang2024enhancing}
Wang, Y., Wang, T., Li, S., He, X., Li, M., Wang, Z., Zheng, N., Shao, B., and
  Liu, T.-Y.
\newblock Enhancing geometric representations for molecules with equivariant
  vector-scalar interactive message passing.
\newblock \emph{Nature Communications}, 15\penalty0 (1):\penalty0 313,
  2024{\natexlab{a}}.

\bibitem[Wang et~al.(2024{\natexlab{b}})Wang, Jiang, Wang, and
  Xuan]{wang2024multi}
Wang, Z., Jiang, T., Wang, J., and Xuan, Q.
\newblock Multi-modal representation learning for molecular property
  prediction: Sequence, graph, geometry.
\newblock \emph{arXiv preprint arXiv:2401.03369}, 2024{\natexlab{b}}.

\bibitem[Wu et~al.(2018)Wu, Ramsundar, Feinberg, Gomes, Geniesse, Pappu,
  Leswing, and Pande]{Wu+2018}
Wu, Z., Ramsundar, B., Feinberg, E.~N., Gomes, J., Geniesse, C., Pappu, A.~S.,
  Leswing, K., and Pande, V.
\newblock {MoleculeNet}: A benchmark for molecular machine learning.
\newblock \emph{Chemical Science}, pp.\  513--530, 2018.

\bibitem[Xiong et~al.(2019)Xiong, Wang, Liu, Zhong, Wan, Li, Li, Luo, Chen,
  Jiang, et~al.]{xiong2019pushing}
Xiong, Z., Wang, D., Liu, X., Zhong, F., Wan, X., Li, X., Li, Z., Luo, X.,
  Chen, K., Jiang, H., et~al.
\newblock Pushing the boundaries of molecular representation for drug discovery
  with the graph attention mechanism.
\newblock \emph{Journal of medicinal chemistry}, 63\penalty0 (16):\penalty0
  8749--8760, 2019.

\bibitem[Xu et~al.(2019{\natexlab{a}})Xu, Luo, and Carin]{xu2019scalable}
Xu, H., Luo, D., and Carin, L.
\newblock Scalable gromov-wasserstein learning for graph partitioning and
  matching.
\newblock \emph{Advances in neural information processing systems}, 32,
  2019{\natexlab{a}}.

\bibitem[Xu et~al.(2019{\natexlab{b}})Xu, Luo, Zha, and Duke]{xu2019gromov}
Xu, H., Luo, D., Zha, H., and Duke, L.~C.
\newblock Gromov-wasserstein learning for graph matching and node embedding.
\newblock In \emph{International conference on machine learning}, pp.\
  6932--6941. PMLR, 2019{\natexlab{b}}.

\bibitem[Xu et~al.(2018)Xu, Li, Tian, Sonobe, Kawarabayashi, and
  Jegelka]{Xu2018}
Xu, K., Li, C., Tian, Y., Sonobe, T., Kawarabayashi, K.-i., and Jegelka, S.
\newblock Representation learning on graphs with jumping knowledge networks.
\newblock In Dy, J. and Krause, A. (eds.), \emph{Proceedings of the 35th
  International Conference on Machine Learning}, volume~80 of \emph{Proceedings
  of Machine Learning Research}, pp.\  5453--5462. PMLR, 10--15 Jul 2018.

\bibitem[Yang et~al.(2019)Yang, Swanson, Jin, Coley, Eiden, Gao, Guzman-Perez,
  Hopper, Kelley, Mathea, Palmer, Settels, Jaakkola, Jensen, and
  Barzilay]{Yang2019}
Yang, K., Swanson, K., Jin, W., Coley, C., Eiden, P., Gao, H., Guzman-Perez,
  A., Hopper, T., Kelley, B., Mathea, M., Palmer, A., Settels, V., Jaakkola,
  T., Jensen, K., and Barzilay, R.
\newblock Analyzing learned molecular representations for property prediction.
\newblock \emph{Journal of Chemical Information and Modeling}, 59\penalty0
  (8):\penalty0 3370–3388, July 2019.
\newblock ISSN 1549-960X.
\newblock \doi{10.1021/acs.jcim.9b00237}.

\bibitem[Zaverkin \& Kästner(2020)Zaverkin and Kästner]{zaverkin:2020}
Zaverkin, V. and Kästner, J.
\newblock Gaussian moments as physically inspired molecular descriptors for
  accurate and scalable machine learning potentials.
\newblock \emph{Journal of Chemical Theory and Computation}, 16\penalty0
  (8):\penalty0 5410--5421, 2020.
\newblock \doi{10.1021/acs.jctc.0c00347}.

\bibitem[Zaverkin et~al.(2024)Zaverkin, Alesiani, Maruyama, Errica,
  Christiansen, Takamoto, Weber, and Niepert]{zaverkin2024higherrank}
Zaverkin, V., Alesiani, F., Maruyama, T., Errica, F., Christiansen, H.,
  Takamoto, M., Weber, N., and Niepert, M.
\newblock Higher-rank irreducible {C}artesian tensors for equivariant message
  passing, 2024.

\bibitem[Zeng et~al.(2023)Zeng, Zhu, Xia, Zeng, and Tong]{zeng2023generative}
Zeng, Z., Zhu, R., Xia, Y., Zeng, H., and Tong, H.
\newblock Generative graph dictionary learning.
\newblock In \emph{International Conference on Machine Learning}, pp.\
  40749--40769. PMLR, 2023.

\bibitem[Zhou et~al.(2023)Zhou, Gao, Ding, Zheng, Xu, Wei, Zhang, and
  Ke]{zhou2023unimol}
Zhou, G., Gao, Z., Ding, Q., Zheng, H., Xu, H., Wei, Z., Zhang, L., and Ke, G.
\newblock Uni-mol: A universal 3d molecular representation learning framework.
\newblock In \emph{The Eleventh International Conference on Learning
  Representations}, 2023.

\bibitem[Zhu et~al.(2023)Zhu, Hwang, Adams, Liu, Nan, Stenfors, Du, Chauhan,
  Wiest, Isayev, Coley, Sun, and Wang]{zhu2023learning}
Zhu, Y., Hwang, J., Adams, K., Liu, Z., Nan, B., Stenfors, B., Du, Y., Chauhan,
  J., Wiest, O., Isayev, O., Coley, C.~W., Sun, Y., and Wang, W.
\newblock Learning over molecular conformer ensembles: Datasets and benchmarks,
  2023.

\end{thebibliography}
\bibliographystyle{icml2024}

\newpage
\appendix
\onecolumn
%%
%\begin{center}
%\textbf{\Large{Supplementary Material for \\ \vspace{.2em} 
%``Structure-Aware 
% E(3)-Invariant Molecular Conformer Aggregation Networks''}}
%\end{center}

\makebox[\textwidth][c]{\textbf{\Large{Structure-Aware E(3)-Invariant Molecular Conformer Aggregation Networks}}}
\makebox[\textwidth][c]{\textbf{\Large{\textit{Supplementary Material}}}}

\tableofcontents
\addtocontents{toc}{\protect\setcounter{tocdepth}{2}}
\vspace{0.2in}
In this supplementary material, we first present rigorous proofs for results concerning the E(3) invariant of the proposed aggregation mechanism in~\cref{sec_theorem_invariance}, while those for the fast convergence of the empirical FGW barycenter are then provided in~\cref{sec_theorem_main}. The entropic FGW algorithm and practical GPU considerations are then given in more detail in~\cref{sec_solving_EFGW}. Finally, some experiment configuration supplements on SchNet neural architecture, 3D conformers generation and comparison between entropic FGW and FGW-mixup are deffered in~\cref{sec_ECS}.

\section{Proof of Theorem \ref{theorem-invariance}}\label{sec_theorem_invariance}

We will proceed as follows. First, we prove that $\Hb^{\mathtt{BC}}$ is invariant to permutations of the input conformers and actions of the group $E(3)$ applied to the input conformers. $\Hb^{\mathtt{BC}}$ is invariant to the order of the input conformers by definition of the barycenter which is invariant to the order of the input graphs. Moreover, since by definition, actions of the group $E(3)$ preserve distances between points in a $3$-dimensional space and, by assumption, the upstream 3D MPNN is invariant to actions of $E(3)$, for any input conformer $S$ and its corresponding graph $G(S) = (\Hb, \mA,  \bsomega)$ and any action $g \in E(3)$ we have that $G(g S) = (\Hb, \mA,  \bsomega) = G(S)$. $\Hb$ is invariant to actions of the group $E(3)$ because the 3D MPNN is invariant to actions of the group. $\mA$ is invariant due to distances between points being invariant. Hence, the input graphs to the barycenter optimization problem are invariant to actions of the group $E(3)$ on the conformers and, therefore, the output barycenters are invariant to such group actions. 

We know now for \cref{eq:aggregate}: $\Hb^{\mathtt{comb}} = \Wb^{\mathtt{2D}}\Hb^{\mathtt{2D}} + \Wb^{\mathtt{3D}}\Hb^{\mathtt{3D}} + \Wb^{\mathtt{BC}}\Hb^{\mathtt{BC}}$, that $\Hb^{\mathtt{BC}}$ is invariant to both actions of the group $E(3)$ and permutations of the input conformers. We also know that $\Hb^{\mathtt{3D}}$ is equivariant to permutations of the input conformers, that is, every permutation of the input conformers also permutes the column of $\Hb^{\mathtt{3D}}$ in the same way. In addition, $\Hb^{\mathtt{3D}}$ is invariant to actions of the group $E(3)$ on the input conformers by the assumption that the 3D MPNN is $E(3)$-invariant.

What remains to be shown is that $\displaystyle \frac{1}{K} \sum_{k=1}^{K} \Hb^{\mathtt{comb}}$ with     $\Hb^{\mathtt{comb}} = \Wb^{\mathtt{2D}}\Hb^{\mathtt{2D}} + \Wb^{\mathtt{3D}}\Hb^{\mathtt{3D}} + \Wb^{\mathtt{BC}}\Hb^{\mathtt{BC}}$ is invariant to column permutations of the matrix $\Hb^{\mathtt{3D}}$. Since we compute the average of the columns of $\Hb^{\mathtt{comb}}$ this is indeed the case.

% \section{Proof of Proposition \ref{proposition_variance_finite}}\label{sec_proposition_variance_finite}

\section{Proof of Theorem \ref{theorem_main}}\label{sec_theorem_main}
We begin by introducing the notation used in the proof of the paper.

\textbf{Undirected attribute graph as Distributions:} Given the set of vertices and edges of the graph $(V,E)$, we define the undirected labeled graphs as tuples of
the form $G = (V,E,\ell_f,\ell_s)$. Here, $\ell_f: V \rightarrow \bsOmega_f$ is a labeling function that associates each vertex $v_i \in V$ with an attribute or feature $\vx_i = \ell_f(v_i)$ in some feature metric space $(\bsOmega_f,d_f)$, and $\ell_s: V \rightarrow \bsOmega_s$ maps a vertex $v_i$ from the graph to its structure representation $\va_k = \ell_s(v_i)$ in some structure space $(\bsOmega_s,\mA)$ specific to each graph where $A:\bsOmega_s\times\bsOmega_s \rightarrow\Rp$ is a symmetric application aimed at measuring similarity between nodes in the graph. 
In our context, it is sufficient to consider the feature space as a $d$-dimensional Euclidean space $\R^{1\times d}$ with Euclidean distance ($\ell^2$ norm), \ie~$(\bsOmega_f,d_f) = (\R^{1\times d},\ell^2)$.
With some abuse, we denote $A$ and $\mA$ as both the measure of structural similarity and the matrix encoding this similarity between nodes in the graph, \ie $\mA[i,k]:=A(\va_i, \va_k)$.

\textbf{The Wasserstein (W) and Gromov-Wasserstein (GW) distances:}
Given two structure graphs $G_1 = (\mH_1,\mA_1,\bsomega_1)$ and $G_2 = (\mH_2,\mA_2,\bsomega_2)$ of order $n_1$ and $n_2$, respectively, described previously by their probability measure $\displaystyle \mu_1 = \sum_{k} \omega_{1k}\delta_{(\vx_{1k},\va_{1k})}$ and $\displaystyle \mu_2 = \sum_{l} \omega_{1l}\delta_{(\vx_{2l},\va_{2l}))}$, we denote $\displaystyle \mu_{\mH_1} = \sum_{k} \omega_k\delta_{\vx_k}$ and $\mu_{\mA_1}=\sum_{k} \omega_k\delta_{\va_k}$ (\resp $\mu_{\mH_2}$ and $\displaystyle \mu_{\mA_2}$) the marginals of $\mu_1$ (\resp $\mu_2$) \wrt the feature and structure, respectively.
We next consider the following notations:
\allowdisplaybreaks
\begin{align}
    J_p(\mA_1,\mA_2,\bspi) &= \sum_{ijkl} L_{ijkl}(\mA_1,\mA_2)^p\bspi_{ij}\bspi_{kl}\\
    \gwp(\mu_{\mH_1},\mu_{\mH_2})^p &= \min_{\bspi \in \bsPi(\bsomega_1,\bsomega_2)} J_p(\mA_1,\mA_2,\bspi)\\
    H_p(\mM,\bspi)& = \sum_{kl} d_f(\vx_{1k},\vx_{2l})^p \bspi_{kl}\\
   \wap(\mu_{\mA_1},\mu_{\mA_2})^p&=  \min_{\bspi \in \bsPi(\bsomega_1,\bsomega_2)} H_{p}(\mM,\bspi).\label{eq_def_Wp}
\end{align}
Note that $\E{p,\alpha}{\mM,\mA_1,\mA_2,\bspi}$ can be further expanded as follows:
\begin{align*}
     \E{p,\alpha}{\mM,\mA_1,\mA_2,\bspi}&=\langle (1-\alpha)\mM^p+\alpha \tL(\mA_1,\mA_2)^p\otimes \bspi,\bspi\rangle\nn\\
   & = \sum_{ijkl} \Big[(1-\alpha) d_f(\vx_{1k},\vx_{2l})^p 
   + \alpha \left| \mA_1(i,k)- \mA_2(j,l)\right|^p\big] \bspi_{ij} \bspi_{kl}
   .
\end{align*}
\textbf{Comparison between FGW and W:} Let $\bspi \in \bsPi(\bsomega_1,\bsomega_2)$ be any admissible coupling between $\bsomega_1$ and $\bsomega_2$. Assume that $\mu_1$ and $\mu_2$ belong to the same ground space $(\bsOmega,\mA,\mu)$, by the definition of the FGW distance in \eqref{eq_FGW_def}, \ie
\begin{align}
    \fgwpa(G_1,G_2) :=\min_{\vpi \in \Pi\left(\bsomega_1, \bsomega_2\right)}\left\langle(1-\alpha) \mM+\alpha \tL\left(\mA_1, \mA_2\right) \otimes \vpi, \vpi\right\rangle,\nn
\end{align}
we get the following important relationship:
\allowdisplaybreaks
\begin{align}
    \fgwpa(G_1,G_2) &\le \left\langle(1-\alpha) \mM+\alpha \tL\left(\mA_1, \mA_2\right) \otimes \vpi, \vpi\right\rangle \nn\\
    &= \sum_{ijkl} \Big[(1-\alpha) d_f(\vx_{1k},\vx_{2l})^p 
   + \alpha \left| \mA[i,k] -\mA[j,l] \right|^p \big]\bspi_{ij} \bspi_{kl}\nn\\
    &\le \sum_{ijkl} \Big[(1-\alpha) d_f(\vx_{1k},\vx_{2l})^p 
   + \alpha \left| \mA[i,j] + \mA[j,k] - \mA[j,k] +\mA[k,l]\right|^p \big]\bspi_{ij} \bspi_{kl}\label{eq_triangle_ineq}\\
    & = \sum_{ijkl} \Big[(1-\alpha) d_f(\vx_{1k},\vx_{2l})^p 
   + \alpha\left|\mA[i,j] + \mA[k,l]  \right|^p \big]\bspi_{ij} \bspi_{kl} \nn\\
   &\le \sum_{ijkl} \Big[(1-\alpha) d_f(\vx_{1k},\vx_{2l})^p 
   + \left(\alpha 2^{p-1} \mA[i,j]^p +\alpha 2^{p-1} \mA[k,l]^p  \right) \big]\bspi_{ij} \bspi_{kl} \label{eq_upper_bound}\\
   &\le  \sum_{ijkl} \Big[\left((1-\alpha) d_f(\vx_{1k},\vx_{2l})^p + \alpha 2^{p-1} \mA[k,l]^p \right)\nn\\
   & \hspace{2cm} + \left((1-\alpha)d_f(\vx_{1i},\vx_{2j})^p+\alpha 2^{p-1} \mA[i,j]^p\right) \big]\bspi_{ij} \bspi_{kl}\nn\\
   &\le  \sum_{kl} \Big[\left((1-\alpha) d_f(\vx_{1k},\vx_{2l})^p + \alpha 2^{p-1} \mA[k,l]^p \right)\Big] \bspi_{kl}\nn\\
   & \hspace{2cm} + \sum_{i,j} \Big[\left((1-\alpha)d_f(\vx_{1i},\vx_{2j})^p+\alpha 2^{p-1} \mA[i,j]^p\right) \big] \bspi_{ij}\nn\\
   &\le \sum_{kl} \Big[\left((1-\alpha) d_f(\vx_{1k},\vx_{2l})^p + 2^{p-1}\alpha \mA[k,l]^p \right)\Big] \bspi_{kl}
   \nn\\
   &\le \sum_{kl} \Big[\left((1-\alpha) d_f(\vx_{1k},\vx_{2l}) + 2^{p-1}\alpha \mA[k,l] \right)\Big]^p \bspi_{kl}
   .\label{eq_inequality_FGW}
\end{align}
Here \eqref{eq_triangle_ineq} is obtained by using the triangle inequality of the metric $\mA$, while \eqref{eq_upper_bound} comes from \cref{lemma_eq_upper_bound}.
Note that the inequality \eqref{eq_inequality_FGW} holds for any admissible coupling $\bspi \in \bsPi(\bsomega_1,\bsomega_2)$. This also holds for the optimal coupling, denoted by $\overline{\bspi}$, for the Wasserstein distance $\wap(\mu_1,\mu_2)$ defined by the following metric space $(\bsOmega,\overline{d})$, where $\overline{d}$ is given by:
\begin{align}
    \overline{d}((\vx_1,\va_1),(\vx_2,\va_2)) = (1-\alpha)d_f(\vx_1,\vx_2)+2^{p-1}\alpha \mA(\va_1,\va_2). \nn
\end{align}
Here, we have to verify that $\overline{d}$ is in fact a distance in $\bsOmega$. Indeed, for the triangle inequality, for any $(\vx_1,\va_1),(\vx_2,\va_2),(\vx_3,\va_3) \in \bsOmega$, we have
\begin{align*}
    \overline{d}((\vx_1,\va_1),(\vx_2,\va_2)) &=  (1-\alpha)d_f(\vx_1,\vx_2)+2^{p-1}\alpha \mA(\va_1,\va_2) \nn\\
    &\le (1-\alpha)d_f(\vx_1,\vx_3)+(1-\alpha)d_f(\vx_3,\vx_2)\nn\\ 
    & \hspace{2cm} +2^{p-1}\alpha \mA(\va_1,\va_2) + 2^{p-1}\alpha \mA(\va_1,\va_3) +2^{p-1}\alpha \mA(\va_3,\va_2)\nn\\
    &= (1-\alpha)d_f(\vx_1,\vx_3) + 2^{p-1}\alpha \mA(\va_1,\va_3)\nn\\
    & \hspace{2cm} +(1-\alpha)d_f(\vx_3,\vx_2)+2^{p-1}\alpha \mA(\va_1,\va_2) +2^{p-1}\alpha \mA(\va_3,\va_2)\nn\\
     &=  \overline{d}((\vx_1,\va_1),(\vx_3,\va_3)) + \overline{d}((\vx_3,\va_3),(\vx_2,\va_2)).
\end{align*}
In this case, the above inequality is derived from the triangle inequalities of $d$ and $C$. The symmetry and equality relation of $\overline{d}$ comes from the same properties of $d_f$ and $\mA$.

By definition of Wasserstein distance in \eqref{eq_def_Wp}, this implies that 
\begin{align} \label{eq_fgw_wp}
     \fgwpa(G_1,G_2) \le \wap(\mu_{\mA_1},\mu_{\mA_2}). 
\end{align}
\begin{lemma}\label{lemma_eq_upper_bound}
    For any $p \in \N$. We have 
    \begin{align}
        (a+b)^p \le 2^{p}(a+b)^p.
    \end{align}
\end{lemma}
\begin{proof}[Proof of \cref{lemma_eq_upper_bound}]
It is easy to check that the inequality is satisfied for $p=1$.
For any $p\in\N$ and $p > 1$, it holds that
\begin{align}
    (x+y)^p &= \left(\left(\frac{1}{2^{p-1}}\right)^{\frac{1}{p}}\frac{x}{\left(\frac{1}{2^{p-1}}\right)^{\frac{1}{p}}} + \left(\frac{1}{2^{p-1}}\right)^{\frac{1}{p}}\frac{y}{\left(\frac{1}{2^{p-1}}\right)^{\frac{1}{p}}}\right)^p\nn\\
    &= \left(\left(\frac{1}{2^{p-1}}\right)^{\frac{1}{p-1}}\frac{x}{\left(\frac{1}{2^{p-1}}\right)} + \left(\frac{1}{2^{p-1}}\right)^{\frac{1}{p-1}}\frac{y}{\left(\frac{1}{2^{p-1}}\right)}\right)^p\nn\\
    &\le \left[\left(\frac{1}{2^{p-1}}\right)^{\frac{1}{p-1}} + \left(\frac{1}{2^{p-1}}\right)^{\frac{1}{p-1}}\right]^{p-1}\left(\frac{x^p}{\frac{1}{2^{p-1}}}+\frac{y^p}{\frac{1}{2^{p-1}}}\right)\nn\\
    &= 2^{p-1}\left[\left(\frac{1}{2^{p-1}}\right)^{\frac{1}{p-1}} \right]^{p-1} 2^{p-1}(x^p+y^p)\nn\\
    &= 2^{p-1}(x^p+y^p).\nn
\end{align}
Here the last inequality is a consequence of the Hölder inequality.
\end{proof}
Recall that we have \begin{align}
    \overline{\mu}_K \in \argmin_{\mu \in \cP_p(\bsOmega)} \frac{1}{K}\sum_{k} \fgwpa^p(\mu,\mu_k) \in \cP_p(\bsOmega)\nn\\
    \overline{\mu}_{0} \in \argmin_{\mu \in \cP_p(\bsOmega)} \int_{\cP_p(\bsOmega)} \fgwpa^p(\mu,\nu) dP(\nu) \in \cP_p(\bsOmega) \subset \cP_p(\bsOmega).\nn
\end{align}
Therefore, $\overline{\mu}_K$ and $\overline{\mu}_{0}$ belong to the same ground space $(\bsOmega,\mA,\mu).$  By using \eqref{eq_fgw_wp}, this implies that 
\begin{align}
    \fgwpa(\overline{\mu}_{0},\overline{\mu}_K) \le 2 \wap(\overline{\mu}_{0},\overline{\mu}_K)^p
\end{align}
and hence
\begin{align}
    \E{}{\fgwtwoa(\overline{\mu}_{0},\overline{\mu}_K)} \le  \E{}{\w_2^2(\overline{\mu}_{0},\overline{\mu}_K)} \le \frac{4\sigma^2_P}{(1-\beta+\gamma)K}.
\end{align}
This is equivalent to the following
 \begin{align}
    \E{}{\fgwtwoa^2(\overline{G}_{0},\overline{G}_K)} \le \frac{4\sigma^2_P}{(1-\beta+\gamma)^2K}.
\end{align}
Here, \cref{lemma_corollary_4_4_le_gouic} leads to the last inquality for the Wassertein distance $\wap(\mu,\nu)$ on the metric space $(\bsOmega,\overline{d})$.

We recall the following definitions and results.
\vspace{0.1in}
\begin{definition}[Strongly convex and smooth functions]\label{def_strong_convex_smooth}
Given a
separable Hilbert space $H$, with inner product $\langle\cdot,\cdot\rangle$  and norm $|\cdot|$, we define the subdifferential $\partial \psi \subset S^2$ of a function $\psi: S \rightarrow\R$ by $\partial \psi = \left\{(x,g): \forall y \in S, \psi(y) \ge \psi(x) + \langle g,y-x\rangle\right\}$ and denote  $\partial \psi (x) = \left\{g \in S: (x,g) \in \partial \psi\right\}$. We then refer to $\psi$ as $\gamma$-strongly convex, if for every $x \in S$ it holds that
\begin{align}
    \partial \psi (x) \neq \emptyset, \text{ and } \langle g,x-y\rangle \ge \psi(x) -\psi(y) + \frac{\alpha}{2}\left|x-y\right|^2 \text{ for all } g \in \partial\psi(x) \text{ and all } y \in S. \label{eq_strongly_convex}
\end{align}
We also recall that a convex function $\psi: S \rightarrow \R$ is called $\beta$-smooth if
\begin{align}
    \langle g_x,x-y\rangle \le \psi(x) - \psi(y) + \frac{\beta}{2}\left|x-y\right|^2, ~ \forall g_x \in \partial \psi(x),~\forall x,y \in S.\label{eq_def_beta_smooth}
\end{align}
\end{definition}
\vspace{0.1in}
\begin{lemma}[Corollary 4.4 from \cite{le_gouic_fast_2022}]\label{lemma_corollary_4_4_le_gouic}
      Let $P \in \cP_2(\cP_2(\bsOmega))$be a probability measure on the 2-Wasserstein space $\w_2$ on the metric space $(\bsOmega,\overline{d})$ and let $\overline{\mu}_{0} \in \cP_2(\bsOmega)$ and $\sigma^2_P$ be a barycenter and a variance functional of $P$, respectively. Let $\gamma,\beta>0$ and suppose that every $\mu \in \supp(P)$ is the pushforward of $\overline{\mu}_{0}$ by the gradient of an $\gamma$-strongly convex and $\beta$ smooth function $\psi_{\overline{\mu}_{0} \rightarrow \mu}$, defined in \cref{def_strong_convex_smooth}, \ie $\mu = (\nabla \psi_{\overline{\mu}_{0} \rightarrow \mu})_{\#}\overline{\mu}_{0}$. If $\beta-\gamma<1$, then $\overline{\mu}_{0}$ is unique and any empirical barycenter $\overline{\mu}_K$ of $P$ satisfies
    \begin{align}
        \E{}{\w_2^2(\overline{\mu}_{0},\overline{\mu}_K)} \le \frac{4\sigma^2_P}{(1-\beta+\gamma)^2K}.
    \end{align}
\end{lemma}

We then obtain the following important identity
\begin{align}
    \E{p,\alpha}{\mM,\mA_1,\mA_2,\bspi} &:= \sum_{ijkl} \Big[(1-\alpha) d_f(\vx_{1k},\vx_{2l})^p + \alpha \left| \mA_1(i,k)- \mA_2(j,l)\right|^p\big] \bspi_{ij} \bspi_{kl}\nonumber\\
    &= (1-\alpha) H_p(\mM,\bspi) + \alpha J_p(\mA_1,\mA_2,\bspi).
\end{align}
Furthermore, given $\bspi_{\alpha}$ as the coupling that minimizes $ \E{p,\alpha}{\mM,\mA_1,\mA_2,\cdot}$, it holds that
\begin{align}
    \fgwpa^p(\mu_1,\mu_2)&=  \min_{\bspi \in \bsPi(\bsomega_1,\bsomega_2)} \E{p,\alpha}{\mM,\mA_1,\mA_2,\bspi}\nn\\
    &= \E{p,\alpha}{\mM,\mA_1,\mA_2,\bspi_{\alpha}} \nn\\
    &=(1-\alpha) H_p(\mM,\bspi_{\alpha}) + \alpha J_p(\mA_1,\mA_2,\bspi_{\alpha})\nn\\
    & \ge (1-\alpha) \wap^p(\mu_{\mA_1},\mu_{\mA_2}) + \alpha\gwp^p(\mu_{\mH_1},\mu_{\mH_2}).
\end{align}
This results in the following by-product:
\begin{align}
    \E{}{\gw_2^2(\overline{\mu}_{0,\mH_1},\overline{\mu}_{K,\mH_2})} \le \frac{4\sigma^2_P}{\alpha(1-\beta+\gamma)^2K},\nn\\
    \E{}{{\w}^2_2(\overline{\mu}_{0,\mA_1},\overline{\mu}_{K,\mA_2})} \le \frac{4\sigma^2_P}{(1-\alpha)(1-\beta+\gamma)^2K}.
\end{align}

\section{Solving Entropic Fused Gromov-Wasserstein}\label{sec_solving_EFGW}
\subsection{Optimization Formulation} 
Entropic-regularization~\cite{cuturi2013sinkhorn} has been well-studied in various OT formulations including entropic Wassterstein~\cite{peyre2019computational, peyre2015entropic} and entropic Gromov-Wasserstein~\cite{rioux2023entropic, Ho_Gromov_closeform} for fast computations of numerous barycenter problems~\cite{cuturi2014fast, peyre_gromov_wasserstein_2016, xu2019gromov, Lin-2020-Revisiting}. However, adapting entropic formulation to the FGW barycenter problem for learning molecular representation, to the best of our knowledge, is novel. Our motivation is to implement Sinkhorn projections solving for the FGW barycenter subgradients, which can be straightforwardly vectorized, computed reversed-mode gradients, and batch-distributed in multi-GPU, benefiting the scaling of the learning pipeline with large molecular datasets.    

Recall that FGW between two graphs $G_1, G_2$ can be described as

\begin{equation} \label{eq:fgw}
\operatorname{FGW}\left(G_1, G_2\right)\equiv \fgwtwoa\left(G_1, G_2\right):=\min_{\vpi \in \Pi\left(\bsomega_1, \bsomega_2\right)}\left\langle(1-\alpha) \mM+\alpha \tL\left(\mA_1, \mA_2\right) \otimes \vpi, \vpi\right\rangle,
\end{equation}

where $\mM := \left( d_f(\mH_1[i], \mH_2[j]) \right)_{n_1 \times n_2} \in \sR^{n_1 \times n_2}$ the pairwise node distance matrix, $\tL\left(\mA_1, \mA_2\right) := \{ L(\mA_1[i,j], \mA_2[k,l]) \}_{ijkl}$ the 4-tensor of structure distance matrix. Assume the loss having the form $L(a, b) = f_1(a) + f_2(b) - h_1(a)h_2(b)$, then from Proposition 1~\cite{peyre2016gromov}, we can write the second term in~\cref{eq:fgw} as

\begin{align}
    \begin{split}
        \tL\left(\mA_1, \mA_2\right) \otimes \vpi &:= \mL - 2h_1(\mA_1) \vpi h_2(\mA_2)^\top, \\
        \mL &:= f_1(\mA_1)\bsomega_1 \boldsymbol{1}_{n_2}^\top + \boldsymbol{1}_{n_1} \bsomega_2^\top f_1(\mA_2)^\top,
    \end{split}
\end{align}

where the square loss $L = \operatorname{L_2}$ having the element-wise functions $f_1(a) = a^2,\,f_2(b) = b^2,\,h_1(a) = a,\,h_2(b) = 2b$, and the KL loss $L = \operatorname{KL}$ having $f_1(a) = a \log{a} - a,\, f_2(b) = b,\, h_1(a) = a,\, h_2(b) = \log{b}$. 
By definition, the entropic FGW distance adds an entropic term as

\begin{equation} \label{eq:efgw_2}
\operatorname{FGW}_{\epsilon}\left(G_1, G_2\right):=\min_{\vpi \in \Pi\left(\bsomega_1, \bsomega_2\right)}\left\langle(1-\alpha) \mM+\alpha \tL\left(\mA_1, \mA_2\right) \otimes \vpi, \vpi\right\rangle - \epsilon \operatorname{H}(\vpi),
\end{equation}

which is a non-convex optimization problem. Following Proposition 2~\cite{peyre2016gromov}, the update rule solving~\cref{eq:efgw_2} is the solution of the entropic OT
\begin{equation} \label{eq:efgw_update}
\vpi = \argmin_{\vpi \in \Pi\left(\bsomega_1, \bsomega_2\right)} \left\langle(1-\alpha) \mM + \mL - 2h_1(\mA_1) \vpi h_2(\mA_2)^\top, \vpi\right\rangle -  \epsilon \operatorname{H}(\vpi),
\end{equation}
where the feature and structure matrices $\mM,\, \mL$ can be precomputed. Since the cost matrix of~\cref{eq:efgw_update} depends on $\vpi$, solving~\cref{eq:efgw_2} involves iterations of solving the linear entropic OT problem~\cref{eq:efgw_update} with Sinkhorn projections, as shown in~\cref{alg:fgw_sinkhorn}.

Following Proposition 4.1 in ~\cite{peyre2019computational}, for sufficiently small regularization $\epsilon$, the approximate solution from the entropic OT problem
$$
\operatorname{OT}_\epsilon(\bsomega_1, \bsomega_2) = \min_{\vpi \in \Pi\left(\bsomega_1, \bsomega_2\right)} \left\langle \mC, \vpi \right\rangle - \epsilon \operatorname{H} (\vpi)
$$ 
approaches the original OT problem. However, small $\epsilon$ incurs serious numerical instability for a high-dimensional cost matrix, e.g., large graph comparisons. In the context of the barycenter problem, too high $\epsilon$ has cheap computation time but leads to a ``blurry'' barycenter solution, while smaller $\epsilon$ produces better accuracy but suffers both numerical instability and computational demanding~\cite{schmitzer2019stabilized, feydy2019interpolating}. Thus, we solve the dual entropic OT problem~\cite{peyre2019computational} 
\begin{equation}
\operatorname{OT}_\epsilon(\bsomega_1, \bsomega_2) \stackrel{\text { def. }}{=} \max _{\vf, \vg} \langle \bsomega_1, \vf \rangle + \langle \bsomega_2, \vg\rangle -\varepsilon\left\langle\bsomega_1 \otimes \bsomega_2, \exp \left(\frac{1}{\varepsilon}(\vf \oplus \vg-\mC)\right)-1\right\rangle,
\end{equation}
where $\vf \in \sR^{n_1}, \vg \in \sR^{n_2}$ are the potential vectors and $\oplus$ is the tensor plus, with stabilized log-sum-exp (LSE) operators~\cite{feydy2019interpolating} for $\forall i \in [1, n_1],\, \forall j \in [1, n_2]$

\begin{equation}
\begin{gathered}
\vf[i]=-\varepsilon \operatorname{LSE}_{k=1}^{n_2}\left(\log \left(\bsomega_2[k]\right)+\frac{1}{\varepsilon} \vg[k]-\frac{1}{\varepsilon} \mC[i,k]\right) \\
\vg[j]=-\varepsilon \operatorname{LSE}_{k=1}^{n_1}\left(\log \left(\bsomega_1[k]\right)+\frac{1}{\varepsilon} \vf[k]-\frac{1}{\varepsilon} \mC[k,j]\right) \\
\text { where } \quad \operatorname{LSE}_{k=1}^{n}\left(\vx[k]\right)=\log \sum_{k=1}^{n} \exp \left(\vx[k]\right)
\end{gathered}
\end{equation}
for numerical stability with large dimension datasets. In practice, we implement these LSEs using \textit{einsum} operations.

The optimal coupling of the dual entropic OT can be computed after the potential vectors converged as
$$
\vpi^* = \exp \left(\frac{1}{\varepsilon}(\vf^* \oplus \vg^* - \mC)\right) \cdot (\bsomega_1 \otimes \bsomega_2).
$$
We state the Sinkhorn algorithm solving the dual entropic OT in~\cref{alg:log_stable_sk}. With~\cref{alg:log_stable_sk}, the auto-differentiation gradient is robust through small perturbation of the potential solutions $\vf^*, \vg^*$. We observe that $\epsilon \in [0.1, 0.2]$ and a few Sinkhorn LSEs are enough for our setting.

\subsection{Empirical Entropic FGW Barycenter} \label{sec_solving_EFGWB}

In our experiments, we propose to solve the entropic relaxation of~\cref{eq_FGWB_def} for utilizing GPU-accelerated Sinkhorn iterations~\cite{peyre2019computational}. Given a set of conformer graphs $\{G_s := (\mH_s, \mA_s,\bsomega_s)\}_{s=1}^K$, we want to optimize the entropic barycenter~\cref{eq:fgw_barycenter}, where we fixed the prior on nodes $\overline{\bsomega}$.~\citet{titouan_optimal_2019} solves~\cref{eq:fgw_barycenter} using Block Coordinate Descent as shown in~\cref{alg:fgw_barycenter}, which iteratively minimizes the original FGW distance between the current barycenter and the graphs $G_s$. In our case, we solve for $K$ couplings of entropic FGW distances to the empirical graphs at each iteration (i.e., $\lambda_s = 1 / K$), then following the update rule for structure matrix (Proposition 4,~\cite{peyre2016gromov})
\begin{align}
\begin{split}
    \overline{\mA}^{(k+1)} &\leftarrow \frac{1}{\overline{\bsomega}~\overline{\bsomega}^\top} \sum_{s=1}^K \lambda_s {\vpi_s^{(k)}} \mA_s {\vpi_s^{(k)}}^\top,\, \textrm{if } L := L_2 \\
    \overline{\mA}^{(k+1)} &\leftarrow \exp \left(\frac{1}{\overline{\bsomega}~\overline{\bsomega}^\top} \sum_{s=1}^K \lambda_s {\vpi_s^{(k)}} \mA_s {\vpi_s^{(k)}}^\top \right), \, \textrm{if } L := \operatorname{KL},
\end{split}
\end{align}
and for the feature matrix~\cite{titouan_optimal_2019, cuturi2014fast}
\begin{equation}
    \overline{\mH}^{(k+1)} \leftarrow \mathrm{diag}(1/\overline{\bsomega}) \sum_{s=1}^K \lambda_s {\vpi_s^{(k)}} \mH_s,
\end{equation}
leading to~\cref{alg:fgw_barycenter}. Note that~\cref{alg:fgw_barycenter} presents only the structure matrix update rule for the square loss $L = L_2$ for clarity. We can modify the structure matrix update rule according to the loss type $L$. In the experiment, we found that the algorithm usually converges after running the number of $10$ outer iterations and $30$ inner iterations. 

\begin{algorithm}[H]
   \caption{Entropic FGW with Sinkhorn projections}
   \label{alg:fgw_sinkhorn}
\begin{algorithmic}
   \STATE {\bfseries Input:} Graph $G_1, G_2$, weighting $\alpha$, entropic scalar $\epsilon$.
   \STATE {\bfseries Optimizing:} $\vpi \in \Pi(\bsomega_1, \bsomega_2)$.
   \STATE Compute $\mL := f_1(\mA_1)\bsomega_1 \boldsymbol{1}_{n_2}^\top + \boldsymbol{1}_{n_1} \bsomega_2^\top f_1(\mA_2)^\top$. 
   \STATE Compute $\mM = \left( d(\mH_1[i], \mH_2[j]) \right)_{n_1 \times n_2}$.
   \STATE Initialize $\vpi$.
   \REPEAT
   \STATE Compute $\mC^{(k)} = (1 - \alpha)\mM + 2\alpha (\mL - h_1(\mA_1) \vpi^{(k)} h_2(\mA_2)^\top)$.
   \STATE Solve $\argmin_{\vpi_s^{(k)}} \left\langle \mC, \vpi \right\rangle - \epsilon \operatorname{H} (\vpi)$ with~\cref{alg:log_stable_sk}.
   \UNTIL{$k$ in \textit{inner iterations} and \textit{not converged}}
\end{algorithmic}
\end{algorithm}
\vspace{-0.2in}
\begin{algorithm}[H]
   \caption{Stabilized LSE Sinkhorn algorithm}
   \label{alg:log_stable_sk}
\begin{algorithmic}
    \STATE {\bfseries Input:} Entropic scalar $\epsilon$, cost matrix $\mC$, marginals $\bsomega_1, \bsomega_2$.
    \STATE Initialize $\vf, \vg = \boldsymbol{0}$.
    \WHILE{termination criteria not met} 
        \FOR{$\forall i \in [1, n]$}
        \STATE $\vf[i]=-\varepsilon \operatorname{LSE}_{k=1}^{m}\left(\log \left(\bsomega_2[k]\right)+\frac{1}{\varepsilon} \vg[k]-\frac{1}{\varepsilon} \mC[i,k]\right)$.
        \ENDFOR
        \FOR{$\forall j \in [1, m]$}
        \STATE $\vg[j]=-\varepsilon \operatorname{LSE}_{k=1}^{n}\left(\log \left(\bsomega_1[k]\right)+\frac{1}{\varepsilon} \vf[k]-\frac{1}{\varepsilon} \mC[k,j]\right)$.
        \ENDFOR
    \ENDWHILE 
    \STATE Return $\vpi^* = \exp \left(\frac{1}{\varepsilon}(\vf^* \oplus \vg^* -\mC)\right) \cdot (\bsomega_1 \otimes \bsomega_2)$.
\end{algorithmic}
\end{algorithm}
\vspace{-0.2in}
\paragraph{Practical GPU considerations.} Our motivation for adopting entropic formulation for FGW barycenter is to solve the barycenter problem fast with (stabilized LSE) Sinkhorn projections, which can be straightforwardly vectorized in PyTorch, facilitating end-to-end unsupervised training with GPU~\cite{cuturi2013sinkhorn, cuturi2014fast, peyre2019computational}. This entropic formulation avoids using Conditional Gradients~\cite{titouan_optimal_2019} to solve FGW, which uses the classical network flow algorithms\footnote{These algorithms are usually available in off-the-shell C++ backend libraries, which are difficult to construct auto-differentiation computation graph over these solvers.} at each iteration. Furthermore, by implementing~\cref{alg:fgw_barycenter} in PyTorch~\cite{paszke2017automatic}, we utilize reverse-mode auto differentiation over solver iterations to propagate gradients from the graph parameters to the barycenter solutions. We observe that the inner entropic OT problem usually converges with a few iterations; thus, we typically limit the number of Sinkhorn iterations solving entropic OT problem to reduce memory burden~\cite{peyre2019computational}.

\vspace{-0.2in}
\paragraph{Scalability and complexity.} As shown in~\cref{alg:fgw_barycenter}, we have three loops to optimize for the FGW barycenter. However, the inner entropic OT problem typically converges with a few stabilized LSE Sinkhorn iterations. Thus, we fix a constant number of Sinkhorn iterations and denote maximum outer (\cref{alg:fgw_barycenter}) and inner iterations (\cref{alg:fgw_sinkhorn}) as $M, N$. In~\cref{alg:fgw_sinkhorn}, the complexity computing $\mC$ is $\gO(n^3 + n^2 d)$ with $n := \max (\{n_s\}_{s=1}^K)$. The first term is the complexity of computing structure cost, while the second is the feature cost complexity. Thus, the complexity for~\cref{alg:fgw_barycenter} is $\gO(MKN(n^3 + n^2d))$ including the feature and structure matrix updates. Note that solving entropic FGW for $K$ graphs can be done in parallel with GPU. Additionally, this complexity does not depend on the maximum edge numbers in graphs $e := \max (\{\| E_s \|\}_{s=1}^K)$, and thus very competitive compared to previous graph matching method~\cite{neyshabur2013netal} for each outer iteration when $e \gg n$.

\section{Experiment Configuration Supplements}\label{sec_ECS}
\subsection{SchNet Neural Architecture}
\label{appendix-schnet-details}
We represent each of the $K$ molecular conformers as a set of atoms $V$ with atom numbers $Z = (Z_1, ..., Z_n)$ and atomic positions $R = (\vr_1, .., \vr_n)$. At each layer $\ell$ an atom $v$ is represented by a learnable representation $\hb_{v}$. We use the geometric message and aggregation functions of SchNet~\citet{Sch+2017} but any other $E(3)$-invariant neural network can be used instead. Besides providing a good trade-off between model complexity and efficacy, we choose SchNet as it was used in prior related work~\cite{axelrod_molecular_2023}.

SchNet relies on the following building blocks. The initial node attributes are learnable embeddings of the atom types, that is, $\hb_{v}^{(0)} \in \mathbb{R}^{d}$ is an embedding of the atom type of node $v$ with $d$ dimensions. 
Two types of combinations of atom-wise linear layers and activation functions
    \vspace{-0.05in}
    \begin{align}
    \varphi_{i}^{(\ell)}\left(\hb\right) &:= \Wb_i^{(l)}\hb + \bb_i^{(l)}
    \ \ \ \ \mbox{ and } \ \ \ \ \ \nonumber\\
    \phi_{i,j}^{(\ell)}\left(\hb\right) &:= \varphi_{j}^{\ell}\left(\mathtt{ssp}\left(\varphi_{i}^{(\ell)}\left(\hb\right) \right)\right)
    \end{align}
    where $\mathtt{ssp}$ is the shifted softplus function (cite), $\Wb_i^{(l)} \in \Rb^{d \times d}$, $\bb_i^{(l)} \in \mathbb{R}^{d}$, with $d$ the hidden dimension of the atom embeddings.
A filter-generating network that serves as a rotationally invariant function $\mathtt{Inv}$:
    \begin{align*}
        \eb_{v, u} &= \mathtt{Inv}\left(\vv{\bbv}_{v}^{(\ell-1)}, \vv{\bbv}_{u}^{\tup{\ell-1}}\right) 
        = \phi_{1,2}^{(\ell)}(\rbf(||\bbr_v - \bbr_u||)),
    \end{align*}
    where $\rbf$ is the radial basis function and $\phi_{1,2}^{(\ell)}$ is a sequence of two dense layers with shifted softplus activation.
    
$E(3)$-invariant message-passing is performed by using the following message function
\vspace{-0.1in}
\begin{align}
    \mb_{v,u}^\tup{\ell} =
    \bM^\tup{\ell}\Bigl(\hb_{v}^\tup{\ell-1},\hb_{u}^\tup{\ell-1}, \eb_{v,u}\Bigr) \nonumber
    = \varphi_{1}^{(\ell)}\left(\hb^{(l - 1)}_u\right)  \circ \ve_{v,u},
\end{align}
where $\circ$ represents the element-wise multiplication. 
The aggregation function is now defined as
\begin{align}
    \bar{\hb}_{v}^{(\ell)} &:= \AGG^\tup{\ell} \bigl(\oms \mb_{v,u}^\tup{\ell} \mid u\in N(v) \cms \bigr) = 
    \sum_{u\in N(v)} \mb_{v,u}^\tup{\ell}.\nonumber
\end{align}

Finally, the update function is given by
\begin{align}
    \hb_{v}^\tup{\ell} &=
	\UPD^\tup{\ell}\Bigl(\hb_{v}^\tup{\ell-1},\AGG^\tup{\ell} \bigl(\oms \mb_{v,u}^\tup{\ell}
	\mid u\in N(v) \cms \bigr)\Bigr) \nonumber\\
        &= \hb_{v}^\tup{\ell-1} + \phi_{3, 4}^{(\ell)}\left(\bar{\hb}_{v}^{(\ell)}\right).
\end{align}

We denote the matrix whose columns are the atom-wise features from the last message-passing layer $L$ with $\Hb$, that is, $\Hb[v] = \hb^{(L)}_v$.

\subsection{Dataset Overview}
\label{sec:dataset_overview}
\textbf{Molecular Property Prediction Tasks}
\ We conduct our experiments on MoleculeNet \citep{Wu+2018}, a comprehensive benchmark dataset for computational chemistry. It spans a wide array of tasks that range from predicting quantum mechanical properties to determining biological activities and solubilities of compounds. In our study, we focus on the regression tasks on four datasets from MoleculeNet benchmark: \texttt{Lipo}, \texttt{ESOL}, \texttt{FreeSolv}, and \texttt{BACE}. 
\begin{itemize}
    \item The \texttt{Lipo} dataset is a collection of 4200 lipophilicity values for various chemical compounds. Lipophilicity is a key property that impacts a molecule's pharmacokinetic behavior, making it crucial for drug development. 
    \item \texttt{ESOL} contains 1128 experimental solubility values for a range of small, drug-like molecules. Understanding solubility is vital in drug discovery, as poor solubility can lead to issues with bioavailability. 
    \item \texttt{FreeSolv} offers both calculated and experimentally determined hydration-free energies for a collection of 642 small molecules. These hydration-free energies are critical for assessing a molecule's stability and solubility in water.
    \item The \texttt{BACE} dataset focuses on biochemical assays related to Alzheimer's Disease. It contains 1513 pIC50 values, indicating the efficiency of various molecules in inhibiting the $\beta$-site amyloid precursor protein cleaving enzyme 1 (BACE-1).
\end{itemize}

\textbf{3D Molecular Classification Tasks}
\ In addition, we evaluate the classification performance using two closely related datasets associated with SARS-CoV: SARS-CoV-2 3CL (\texttt{CoV-2 3CL}), and SARS-CoV-2 (\texttt{CoV-2}).
\begin{itemize}
    \item \texttt{CoV-2} 3CL protease dataset comprises 76 instances corresponding to inhibitory interactions, considering a total of 804 unique species. This dataset specifically addresses the inhibition of the SARS-CoV-2 3CL protease (denoted as `CoV-2 CL') \citep{diamondXChem}.
    \item \texttt{CoV-2} dataset, which encompasses 92 instances across a spectrum of 5,476 unique species. This dataset focuses on the broader context of inhibitory interactions against SARS-CoV-2 measured in vitro within human cells \citep{PPR:PPR153247, Touret2020}.
\end{itemize}

\textbf{Reaction-level molecule properties prediction}
The \texttt{BDE} dataset \cite{meyer2018machine} contains 5915 organometallic catalysts ($ML_{1}L_{2}$), with metal centers (Pd, Pt, Au, Ag, Cu, Ni) and two flexible organic ligands ($L_1$ and $L_2$) chosen from a 91-ligand library. It includes conformations of each unbound catalyst and those bound to ethylene and bromide after reacting with vinyl bromide (resulting in 11830 individual molecules). The dataset provides electronic binding energies, calculated as the energy difference between the bound-catalyst complex and the unbound catalyst, optimized using DFT. Conformers are initially generated with Open Babel and then geometry-optimized to likely represent the global minimum energy structures at the force field level.
\subsection{3D Conformers Generation}
RDKit offers two methodologies to generate conformers for molecules:
\begin{itemize}
    \item The distance geometry approach employs distance geometry principles for conformer generation, starting with the determination of a molecule's distance bounds matrix based on connectivity and predefined rules. This matrix is then refined and used to formulate a random distance matrix, which subsequently guides the molecule's embedding into 3D space. The resulting atomic coordinates undergo further refinement through a specialized “distance geometry force field.”
    \item ETKDG method, which refines generated conformers by integrating torsion angle preferences from the Cambridge Structural Database (CSD). This technique can be further enhanced with additional torsion terms, catering especially to small rings and macrocycles, yielding high-quality conformers suitable for direct application in many scenarios.
\end{itemize} 

In our experiments, we applied a standardized approach to configuring all benchmark datasets, encompassing the following steps:
\begin{itemize}

\item Conformer Generation: During the training phase, we use RDKit to generate a fixed set of 200 conformers for every molecular structure specified by its SMILES string. However, in each epoch, each molecular is sampled with a $K$ conformers ($K << 200$). For the validation and testing, we use a fixed seed and generate randomly $K$ conformers for each sample in the dataset. 

\item Parallel Processing: Utilizing a process pool enhances the parallelization of conformer generation, thereby optimizing overall efficiency. We provide in Table \ref{tab:summary_edges_nodes_graph} the average execution time for generating a single conformer from its SMILES string across diverse datasets.
\end{itemize}

For a comprehensive 3D structural analysis, we present summary statistics detailing the number of edges and nodes (Table \ref{tab:summary_edges_nodes_graph}). These statistics provide insights into the structural characteristics of molecules within the datasets. Average values offer a perspective on the typical size of molecules in terms of edges and nodes, while minimum and maximum values reflect the varying complexities of molecular structures across datasets. Notably, the Lipo and BACE datasets emerge as the most intricate graphs, contrasting with ESOL and FreeSolv, which exhibit sparser structures. We illustrate in Figure \ref{fig:cfm_vis} some typical generated conformers for each dataset.

\begin{table}[h]
\caption{Summary statistics for edge and node counts in diverse datasets, reflecting the runtime needed to generate a conformer from a molecular structure.
}
\label{tab:summary_edges_nodes_graph}
\centering
\scalebox{0.9}{
\begin{tabular}{cccccccc}
\toprule
\multirow{2}{*}{\textbf{Dataset}} & \multicolumn{3}{c}{\textbf{Number of Edges}} & \multicolumn{3}{c}{\textbf{Number of Nodes}} & \multirow{2}{*}{\textbf{Execution Time (seconds)}} \\
% \hline
\cline{2-7}
 & Avg & Min & Max & Avg & Min & Max \\
\midrule
Lipo & 101.8 & 24 & 412 & 48.4 & 12 & 203 & \(4.68 \times 10^{-6}\) \\
ESOL & 52 & 6 & 252 & 25.6 & 4 & 119 & \(3.58 \times 10^{-6}\) \\
FreeSolv & 35.5 & 4 & 92 & 18.1 & 3&44 & \(3.13 \times 10^{-6}\)\\
BACE & 135 &	36 & 376 & 64.7 &	17 & 184 & \(4.34 \times 10^{-6}\) \\
\midrule
CoV-2 3CL & 56 & 16 & 96 & 27.4 & 8 & 48 & \(3.12 \times 10^{-6}\) \\
CoV-2& 95.2 &	4 & 220 & 45.7 & 3 & 100 & \(3.96 \times 10^{-6}\) \\
\bottomrule
\end{tabular}}

\end{table}

\subsection{Ablation Studies of Number of Conformers}
\begin{table}[H]
\caption{The impact of number of conformations $K$ on the accuracy of the $\conan$ model (without barycenter). Results are in MSE $\downarrow$ computed on the validation set. \textbf{Bold} and underline values denote first and second-rank results.}
\vspace{0.05in}
\label{tab:k-choice}
\centering
\resizebox{0.6\columnwidth}{!}{%
\begin{tabular}{lllll}
    \toprule
    \multicolumn{1}{c}{\bf $K$}  &\multicolumn{1}{c}{\bf Lipo} &\multicolumn{1}{c}{\bf ESOL} &\multicolumn{1}{c}{\bf FreeSolv} &\multicolumn{1}{c}{\bf BACE} \\
    \cline{1-5}
    0 & $1.387 \pm 0.206$  & $2.288 \pm 0.017$ & $8.564 \pm 1.345$ & $1.844 \pm 0.33$\\
    1 & $0.619 \pm 0.045$ & $0.645 \pm 0.054$ & $2.306 \pm 0.807$ & $0.705 \pm 0.064$ \\
    3 & $0.581 \pm 0.033$ & $0.592 \pm 0.072$ & $2.035 \pm 0.256$ & \underline{$0.653 \pm 0.026$} \\
    5 & \underline{$0.567 \pm 0.019$} & \boldsymbol{$0.581 \pm 0.051$} & $1.799 \pm 0.662$ & \boldsymbol{$0.616 \pm 0.051$} \\
    10 & \boldsymbol{$0.564 \pm 0.030$} & \underline{$0.583 \pm 0.027$} & \boldsymbol{$1.568 \pm 0.183$} & $0.832 \pm 0.143$ \\
    20 & $0.569 \pm 0.003$ & {$0.589 \pm 0.012$} & \underline{$1.742 \pm 0.143$} & $0.670 \pm 0.036$ \\
    \bottomrule
\end{tabular}
}
\end{table}
Table \ref{tab:k-choice} illustrates that incorporating 3D conformers with $K \geq 1$ significantly enhances performance compared to relying solely on 2D molecular graphs, as used in the 2D-GAT model. However, the relationship between the number of conformations and model accuracy is not linear or straightforward. For instance, while increasing the number of conformations to $K=10$ improves performance for datasets such as Lipo and ESOL, the best overall performance is usually achieved with $K=5$. This suggests that an optimal number of conformers maximizes model accuracy, which varies depending on the specific dataset.

\subsection{Entropic FGW versus FGW-Mixup detail}
\label{fgwmixup-details}
\vspace{-0.05in}
We provide more details on the efficiency ablation study in~\cref{subsection-efficiency}. We adapt the original GitHub repository~\url{https://github.com/ArthurLeoM/FGWMixup} from~\citet{ma2023fused} as the baseline. In the context of $K$ FGW barycenter problem, due to the numerical instability of the $\exp$ function, we have to set small stepsize $\gamma$ of the Bregman projections (Algorithm 2 in~\cite{ma2023fused}) to avoid NAN values output of FGW-Mixup in some datasets, leading to more inner iterations to converge.
Indeed, it is particularly difficult to find optimal parameters for FGW-Mixup, balancing between the marginal errors inducing the FGW subgradient noise at the outer iteration and the empirical convergence rate at the inner iteration.

\textbf{Running Time Analysis.}\ \  In Figure \ref{fig:runtime_freesolv_cov23cl}, we compare the running time of our solver with FGW-Mixup on two datasets, FreeSolv, and CoV-2 3CL, for both \textit{forward and backward steps} to update gradients for the whole models. We measure average times over epochs during the training steps with increasing values of conformers $K$. Note that in FGW-Mixup, the solver is not supported for inference on GPU, while our algorithm is designed for this purpose and can be scaled on large training samples using data distributed parallel in Pytorch. In particular, \textit{\conan-FGW Single-GPU \conan-FGW on Multi-GPUs} indicates the version where one and four Tesla V100-32GB are used for training, respectively. 

To delve deeper into the computation of the FGW barycenter, we present the runtime analysis in Figure \ref{fig:runtime_freesolv_cov23cl_fw_only}. The configuration mirrors that of Figure \ref{fig:runtime_freesolv_cov23cl}, with the exception that the runtime is specifically gauged at the barycenter components during the forward step. Notably, the execution time exhibits a consistent pattern comparable to Figure \ref{fig:runtime_freesolv_cov23cl}, highlighting that \conan-FGW outperforms FGX-Mixup in both single GPU and multi-GPU setups, achieving significantly faster runtimes as the number of conformers is scaled.
\vspace{-0.1in}
\begin{figure}[H]
\centering
\includegraphics[width=0.6\columnwidth]{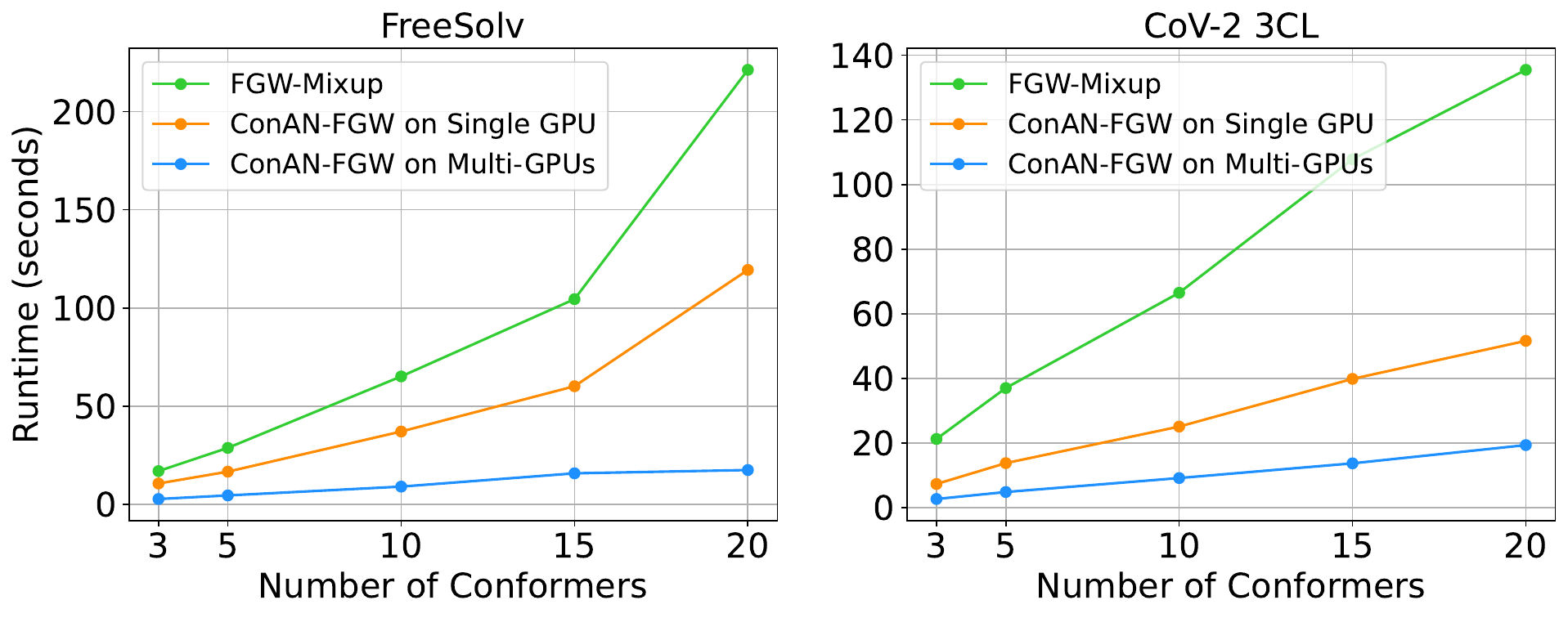}
\vspace{-0.15in}
\caption{\centering Runtime comparison of FGW-Mixup, \conan-FGW (single and multi-GPU) in the FGW barycenter computation.}
\label{fig:runtime_freesolv_cov23cl_fw_only}
\end{figure}

\iffalse
\begin{figure}[H]
\centering
\includegraphics[width=0.6\columnwidth]{ICML2024/figures/freesolv_cov23cl_fw_only.pdf}
\vspace{-0.15in}
\caption{Runtime comparison of FGW-Mixup and \conan-FGW (single and multi-GPU) in the FGW barycenter computation.}
\label{fig:runtime_freesolv_cov23cl_fw_only}
\end{figure}
\fi
\vspace{-0.2in}
\textbf{Error Analysis.}\ \ In this part, we investigate the error of \conan-FGW and FGW-Mixup. To this end, we use the solution of the original FGW problem solved by the Conditional Gradient algorithm~\cite{titouan_fused_2020} as the approximated ground truth for comparing solution errors (Table \ref{tab:error_sol_solvers}). We fix the same hyperparameters for both solvers as in~\cref{fig:runtime_freesolv_cov23cl}. As expected, the FGW-Mixup solution errors are slightly smaller than our \conan-FGW ones. This is due to the fact that (i) to prevent numerical instability, we set small stepsize for the mirror descents (i.e., alternating Bregman projections) and (ii) FGW-Mixup asymptotically converges to the original FGW solution up to a bounded gap~\cite{ma2023fused}. However, this induces more computational time for large FGW problems, as seen~\cref{fig:runtime_freesolv_cov23cl,fig:runtime_freesolv_cov23cl_fw_only}. In contrast, \conan-FGW maintains comparable solution errors to FGW-Mixup while having reasonable computational runtime and being compatible with deploying multi-GPU for large-scale problems.

\begin{figure}[!t]
\centering
\includegraphics[width=0.8\columnwidth]{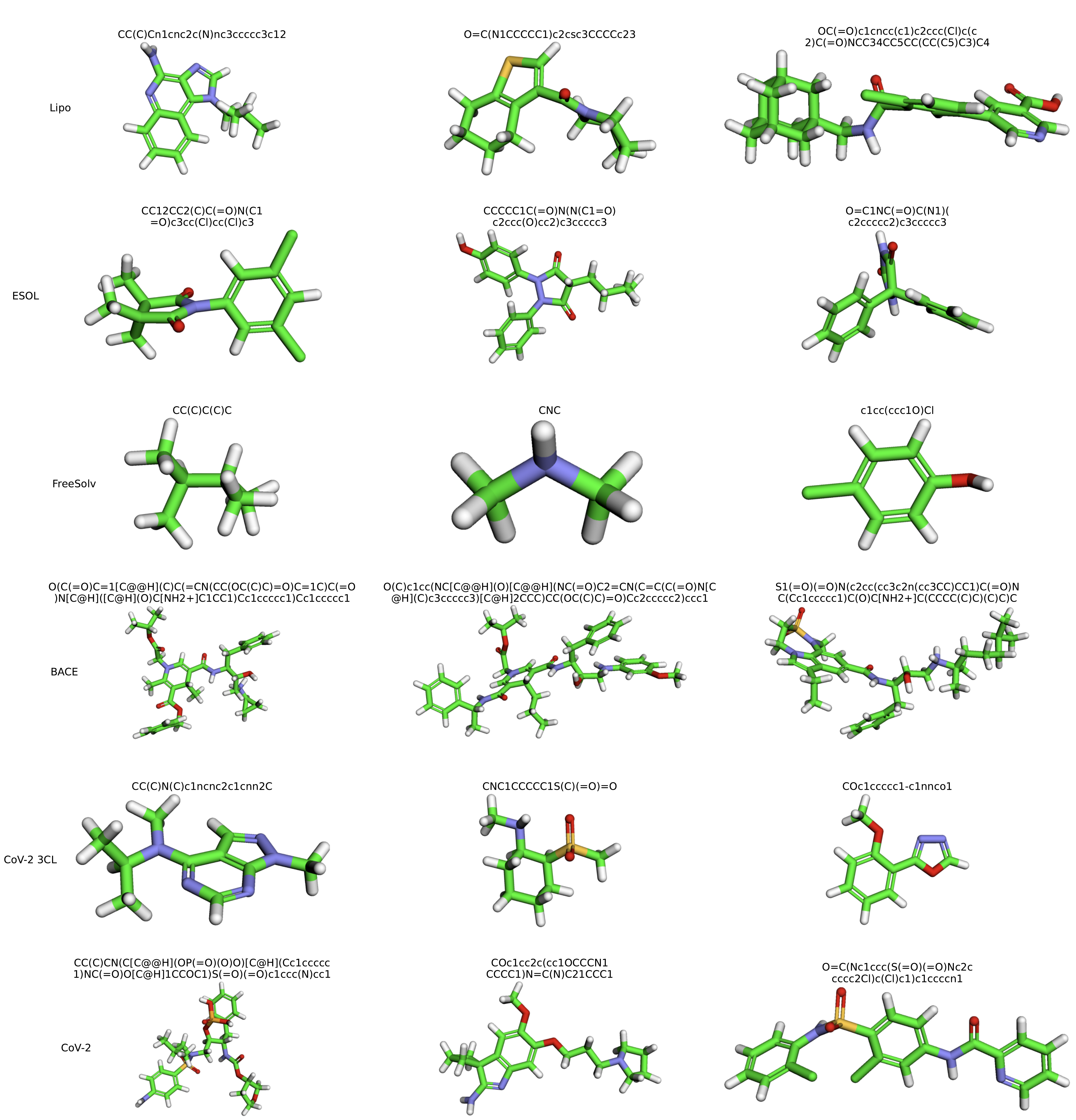}
      \caption{Visualizing 3D molecular conformers with corresponding SMILES strings across diverse datasets.}
      \label{fig:cfm_vis}
\end{figure}

\begin{table}[!t]
\centering
\caption{Error estimation performance across datasets, demonstrating the influence of conformer variations and different methodologies for Ground Truth (GT) in conjunction with \conan-FGW and FGW-Mixup. The comparing matrix metrics are Normalized Frobenius norm, Mean Absolute Error (MAE), Mean Absolute Percent Error (MAPE), and Mean Square Error (MSE).
}
\vspace{0.1in}
\scalebox{0.8}{
\begin{tabular}{cccccccccc}
\toprule
\multirow{2}{*}{\textbf{Dataset}} & \multirow{2}{*}{\textbf{Conformers}} & \multicolumn{4}{c}{\textbf{GT and \conan-FGW}} & \multicolumn{4}{c}{\textbf{GT and FGW-Mixup}} \\
% \hline
\cline{3-10}
 & & N-Frobenius & MAE & MAPE & MSE & N-Frobenius & MAE & MAPE & MSE \\
\midrule
\multirow{5}{*}{FreeSolv} & 3 & 0.1325 & 0.1727 & 0.3523 & 0.0812 & 0.1190 & 0.1590 & 0.3210 & 0.0671 \\
 & 5 & 0.1387 & 0.1823 & 0.3753 & 0.0870 & 0.1258 & 0.1695 & 0.3466 & 0.0731 \\
 & 10 & 0.1431 & 0.1874 & 0.3876 & 0.0919 & 0.1323 & 0.1776 & 0.3638 & 0.0792 \\
 & 15 & 0.1460 & 0.1924 & 0.3980 & 0.0947 & 0.1358 & 0.1819 & 0.3703 & 0.0832 \\
 & 20 & 0.1453 & 0.1920 & 0.3954 & 0.0952 & 0.1336 & 0.1805 & 0.3662 & 0.0816 \\
\midrule
\multirow{5}{*}{CoV-2 3CL} & 3 & 0.0859 & 0.1696 & 0.4207 & 0.0670 & 0.0804 & 0.1626 & 0.4055 & 0.0600 \\
 & 5 & 0.0842 & 0.1688 & 0.4114 & 0.0632 & 0.0793 & 0.1616 & 0.3942 & 0.0569 \\
 & 10 & 0.0879 & 0.1801 & 0.4452 & 0.0719 & 0.0806 & 0.1697 & 0.4201 & 0.0637 \\
 & 15 & 0.0859 & 0.1729 & 0.4251 & 0.0670 & 0.0764 & 0.1571 & 0.3899 & 0.0543 \\
 & 20 & 0.0902 & 0.1823 & 0.4558 & 0.0714 & 0.0865 & 0.1779 & 0.4460 & 0.0653 \\
\bottomrule
\end{tabular}
\label{tab:error_sol_solvers}}
\end{table}

\subsection{Visualize Conformers Generated by RDKit}
We present in Figure\,\ref{fig:cfm_vis} typical 3D conformers generated by RDKit with their string inputs denote below each figure.

\end{document}